\documentclass[12pt]{l4dc2022}

\title[Model Selection for Generic Reinforcement Learning]{Model Selection for Generic Reinforcement Learning}
\usepackage{times}

\usepackage{algorithm}
\usepackage{algorithmic}

\usepackage{enumerate,enumitem}

\usepackage{amssymb}
\usepackage{comment}
\usepackage{mathtools}
\usepackage{graphicx}

\newtheorem{theorem}{Theorem}
\newtheorem{remark}{Remark}

\newtheorem{definition}{Definition}
\newtheorem{lemma}{Lemma}
\newtheorem{proposition}{Proposition}
\newtheorem{assumption}{Assumption}

\usepackage{comment}
\usepackage{epsf}
\usepackage{graphicx}

\usepackage{color}

\usepackage{mathtools}
\usepackage{amsfonts}
\usepackage{amsmath, bm}
\usepackage{amssymb}

\usepackage{pgfplots}
\pgfplotsset{width=5.75cm,compat=1.9}

\usepackage{textcomp}
\usepackage{xcolor}

\DeclareSymbolFont{extraup}{U}{zavm}{m}{n}
\DeclareMathSymbol{\vardiamond}{\mathalpha}{extraup}{87}

\long\def\comment#1{}

\newcommand{\poly}{\mathsf{poly}}
\newcommand{\polylog}{\mathsf{polylog}}

\newcommand{\norms}[1]{\left\|#1 \right\|}

\newcommand{\eucnorm}[1]{\left\|#1 \right\|}

\newcommand{\inprod}[2]{\ensuremath{\langle #1 , \, #2 \rangle}}

\newcommand{\E}{\ensuremath{{\mathbb{E}}}}

\newcommand{\Prob}{\ensuremath{{\mathbb{P}}}}

\DeclareMathOperator{\Var}{var}

\DeclareMathOperator*{\argmin}{argmin}
\DeclareMathOperator*{\argmax}{argmax}

\DeclareMathOperator{\dimn}{dim}

\newcommand{\real}{\ensuremath{\mathbb{R}}}

\newcommand{\R}{\ensuremath{\mathbb{R}}}

\newcommand{\cA}{\mathcal{A}}
\newcommand{\cB}{\mathcal{B}}

\newcommand{\cD}{\mathcal{D}}
\newcommand{\cE}{\mathcal{E}}
\newcommand{\cF}{\mathcal{F}}
\newcommand{\cG}{\mathcal{G}}

\newcommand{\cS}{\mathcal{S}}
\newcommand{\cP}{\mathcal{P}}
\newcommand{\cV}{\mathcal{V}}
\newcommand{\cM}{\mathcal{M}}
\newcommand{\cN}{\mathcal{N}}

\newcommand{\cL}{\mathcal{L}}
\newcommand{\cX}{\mathcal{X}}
\newcommand{\cZ}{\mathcal{Z}}
\newcommand{\cO}{\mathcal{O}}

\author{%
 \Name{Avishek Ghosh} \Email{a2ghosh@ucsd.edu}\\
 \addr UCSD
 \AND
 \Name{Sayak Ray Chowdhury} \Email{sayak@bu.edu}\\
 \addr Boston University%
 \AND
 \Name{Kannan Ramachandran} \Email{kannanr@eecs.berkeley.edu}\\
 \addr UC Berkeley%
}

\begin{document}

\setlength{\belowdisplayskip}{3pt} \setlength{\belowdisplayshortskip}{3pt}
\setlength{\abovedisplayskip}{3pt} \setlength{\abovedisplayshortskip}{3pt}

\maketitle
\vspace{-4mm}
\begin{abstract}%
 We address the problem of model selection for the finite horizon episodic Reinforcement Learning (RL) problem where the transition kernel $P^*$ belongs to a family of models $\mathcal{P}^*$ with finite metric entropy. In the model selection framework, instead of $\mathcal{P}^*$, we are given $M$ nested families of transition kernels $\cP_1 \subset \cP_2 \subset \ldots \subset \cP_M$. We propose and analyze a novel algorithm, namely \emph{Adaptive Reinforcement Learning (General)} (\texttt{ARL-GEN}) that adapts to the smallest such family where the true transition kernel $P^*$ lies. \texttt{ARL-GEN} uses the Upper Confidence Reinforcement Learning (\texttt{UCRL}) algorithm with value targeted regression as a blackbox and puts a model selection module at the beginning of each epoch. Under a mild separability assumption on the model classes, we show that \texttt{ARL-GEN} obtains a regret of $\Tilde{\mathcal{O}}(d_{\mathcal{E}}^*H^2+\sqrt{d_{\mathcal{E}}^* \mathbb{M}^* H^2 T})$, with high probability, where $H$ is the horizon length, $T$ is the total number of steps, $d_{\mathcal{E}}^*$ is the Eluder dimension and $\mathbb{M}^*$ is the metric entropy corresponding to $\mathcal{P}^*$. Note that this regret scaling matches that of an oracle that knows $\mathcal{P}^*$ in advance. We show that the cost of model selection for \texttt{ARL-GEN} is an additive term in the regret having a weak dependence on $T$. Subsequently, we remove the separability assumption and consider the setup of linear mixture MDPs, where the transition kernel $P^*$ has a linear function approximation. With this low rank structure, we propose novel adaptive algorithms for model selection, and obtain (order-wise) regret identical to that of an oracle with knowledge of the true model class.
\end{abstract}

\textbf{Full paper:} The full paper is available at https://tinyurl.com/3tfrzcju

\section{INTRODUCTION}
\label{sec:intro}
A Markov decision process (MDP) \citep{puterman2014markov} is a classical framework to model a reinforcement learning (RL) environment, where an agent interacts with the environment by taking successive decisions and observe rewards. One of the objectives in RL is to maximize the total reward accumulated over multiple rounds, or equivalently minimize the \emph{regret} in comparison with an optimal policy \citep{jaksch2010near}. Regret minimization is useful in several sequential decision-making problems such as portfolio allocation and sequential investment, dynamic resource allocation in communication systems, recommendation systems, etc. In these settings, there is
no separate budget to purely explore the
unknown environment; rather, exploration and exploitation need to be carefully balanced. 

In many applications (e.g., AlphaGo, robotics), the space of states and actions can be very large or even infinite, which makes RL challenging, particularly in generalizing learnt
knowledge across unseen states and actions.
In recent years, we have witnessed an explosion in the RL literature to tackle this challenge, both in theory (see, e.g., \cite{osband2014model,chowdhury2019online,ayoub2020model,wang2020provably,kakade2020information}), and in practice (see, e.g., \cite{mnih2013playing,williams2017model}). The most related work to ours is by \citet{ayoub2020model}, which proposes an algorithm, namely \texttt{UCRL-VTR}, for model-based
RL without any structural assumptions, and it is based on the upper confidence RL and value-targeted regression principles. The regret of \texttt{UCRL-VTR} depends on the \emph{eluder dimension} \citep{russo2013eluder} and the \emph{metric entropy} of the corresponding family of distributions $\cP$ in which the unknown transition model $P^*$ lies. In most practical cases, however, the class $\mathcal{P}$ given to (or estimated by) the RL agent is quite pessimistic; meaning that $P^*$ actually lies in a small subset of $\mathcal{P}$ (e.g., in the game of Go, the learning is possible without the need for visiting all the states \citep{silver2017mastering}).
This issue becomes more interpretable in the setup where $P^*$ assumes a low-rank structure via a linear model (see, e.g., \cite{jin2019provably,yang2019reinforcement,jia2020model,zhou2020nearly}). RL problems with linear function approximation are often parameterized by a $\theta^* \in \real^d$. In this setting, the regret usually depends on (a) the norm $\|\theta^*\|$ and (b) the dimension $d$. In particular, \citet{jia2020model} show that the \texttt{UCRL-VTR} algorithm, when instantiated in the linear setting, achieves a regret of $\mathcal{O}(bd \sqrt{H^3 T)}$, where $b$ is an upper bound over $\|\theta^*\|$, $H$ is the horizon length and $T$ is the total number of steps.
In this setting, the choice of $\|\theta^*\|$ and dimension or sparsity of $\theta^*$ is crucial. If these quantities are under-specified, the regret bounds may fail to hold, and the learning algorithms may incur linear regret. Furthermore, if these quantities are over-specified (which is the case in most RL applications, e.g., the linear quadratic regulators \citep{abbasi2011regret}), the regret bounds are unnecessarily large. Hence, one needs algorithms that exploit the structure of the problem and adapt to the problem complexity with high confidence.

The problem of model selection can be formally stated as follows -- we are given a family of $M$ nested hypothesis classes $\mathcal{P}_1 \subset \mathcal{P}_2 \subset \ldots \subset \mathcal{P}_M $, where each class posits a plausible model class for the underlying RL problem.
The true model $P^*$ lies in a model class $\cP_{m*}$, which is assumed to be contained in the family of nested classes. Model selection guarantees refer to algorithms whose regret scales in the complexity of the \emph{smallest model class containing the true model $P^*$}, even though the algorithm is not aware of that a priori. 
Model selection is well studied in the contextual bandit setting. In this setting, minimax optimal regret guarantees can be obtained by exploiting the structure of the problem along with an eigenvalue assumption \citep{osom, foster_model_selection, ghosh2021problem}. In this work, we address the problem of model selection in RL environments. We consider both the setups where the underlying transition distribution has no structural assumption as well as when it admits a low rank linear function approximation. In the RL framework, the question of model selection has received little attention. In a series of works, \citet{pacchiano2020regret,aldo} consider the corralling framework of \cite{corral} for
 contextual bandits and reinforcement learning. While the corralling framework is versatile, the price for this is that the cost of model selection is multiplicative rather than additive. In particular, for the special case of linear bandits and linear reinforcement learning, the regret scales as $\sqrt{T}$ in time with an additional multiplicative factor of $\sqrt{M}$, while the regret scaling with time is strictly larger than $\sqrt{T}$ in the general contextual bandit. These papers treat all the hypothesis classes as bandit arms, and hence work in a (restricted) partial information setting, and as a consequence explore a lot, yielding worse regret. On the other hand, we consider all $M$ classes at once (full information setting) and do inference, and hence explore less and obtain lower regret.

Very recently, \citet{vidya} study the problem of model selection in RL with function approximation. Similar to the \emph{active-arm elimination} technique employed in standard multi-armed bandit (MAB) problems \citep{eliminate}, the authors eliminate the model classes that are dubbed misspecified, and obtain a regret of $\mathcal{O}(T^{2/3})$. On the other hand, our framework is quite different in the sense that we consider model selection for RL with \emph{general} transition structure. Moreover, our regret scales as $\mathcal{O}(\sqrt{T})$. Note that the model selection guarantees we obtain in the linear MDPs are partly influenced by \cite{ghosh2021problem}, where model selection for linear contextual bandits are discussed. However, there are a couple of subtle differences: (a) for linear contextual framework, one can perform pure exploration, and \cite{ghosh2021problem} crucially leverages that and (b) the contexts in linear contextual framework is assumed to be i.i.d, whereas for linear MDPs, the contexts are implicit and depend on states, actions and transition probabilities.

\paragraph{Outline and Contributions:}
The first part of the paper deals with the setup where we consider \emph{any general} model class that are totally bounded, i.e., for arbitrary precision, the metric entropy is bounded. Notice that this encompasses a significantly larger class of problems compared to the problems with function approximation. Assuming a nested family of transition kernels, we propose an adaptive algorithm, namely \emph{Adaptive Reinforcement Learning-General} (\texttt{ARL-GEN}). Assuming the transition families are well-separated, \texttt{ARL-GEN} constructs a test statistic and thresholds it to identify the correct family. We show that this \emph{simple} scheme achieves the regret $\Tilde{\mathcal{O}}(d_{\mathcal{E}}^*H^2+\sqrt{d_{\mathcal{E}}^* \mathbb{M}^* H^2 T})$, where $d_{\mathcal{E}}^*$ is the \emph{eluder dimension} and $\mathbb{M}^*$ is the \emph{metric entropy} corresponding to the transition family $\cP_{m^*}$ in which the true model $P^*$ lies. The regret bound shows that \texttt{ARL-GEN} adapts to the true problem complexity, and the cost of model section is only $\mathcal{O}(\log T)$, which is minimal compared to the total regret.
In the second part of the paper, we focus on the linear function approximation framework, where the transition kernel is parameterized by a vector $\theta^* \in \R^d$. With norm ($\|\theta^*\|$) and sparsity ($\|\theta^*\|_0$) as problem complexity parameters, we propose two algorithms, namely \emph{Adaptive Reinforcement Learning-Linear (norm)} and \emph{Adaptive Reinforcement Learning-Linear (dim)}, respectively, that adapt to these complexities -- meaning that the regret depends on the actual problem complexities $\|\theta^*\|$ and $\|\theta^*\|_0$. Here also, the costs of model selection are shown to be \emph{minor lower order} terms.

\subsection{Related Work}
\paragraph{Model Selection in Online Learning:} Model selection for bandits are only recently being studied \citep{ghosh2017misspecified,osom}. These works aim to identify whether a given problem instance comes from contextual or standard setting. For linear contextual bandits, with the dimension of the underlying parameter as a complexity measure, \citet{foster_model_selection,ghosh2021problem} propose efficient algorithms that adapts to the \emph{true} dimension of the problem. While \cite{foster_model_selection} obtains a regret of $\mathcal{O}(T^{2/3})$, \cite{ghosh2021problem} obtains a $\mathcal{O}(\sqrt{T})$ regret (however, the regret of \cite{ghosh2021problem} depends on several problem dependent quantities and hence not instance uniform). Later on, these guarantees are extended to the generic contextual bandit problems without linear structure \citep{ghosh2021modelgen,krishnamurthy2021optimal}, where $\mathcal{O}(\sqrt{T})$ regret guarantees are obtained. The algorithm {\ttfamily Corral} was proposed in \cite{corral}, where the optimal algorithm for each model class is casted as an expert, and the forecaster obtains low regret with respect to the best expert (best model class). The generality of this framework has rendered it fruitful in a variety of different settings; see, for example \cite{corral, arora2021corralling}.


\paragraph{RL with Function Approximation:} 
Regret minimization in RL under function approximation is first considered in 
\cite{osband2014model}. It makes explicit model-based assumptions and the regret bound depends on the eluder dimensions of the models. In contrast, \cite{yang2019reinforcement} considers a low-rank linear transition model and propose a model-based algorithm with regret $\mathcal{O}(\sqrt{d^3H^3T})$. Another line of work parameterizes the $Q$-\emph{functions} directly, using state-action
feature maps, and develop model-free algorithms with regret $\mathcal{O}(\texttt{poly}(dH)\sqrt{T})$ bypassing the need for fully
learning the transition model \citep{jin2019provably,wang2019optimism,zanette2020frequentist}. A recent line of work \citep{wang2020provably,yang2020provably} generalize these approaches by designing algorithms that work with general and neural function approximations, respectively.




\subsection{Preliminaries}

\paragraph{Notation:} For any $n \in \mathbb{N}$, $[n]$ denote the set of integers $\{1,2,\ldots,n\}$. $\gamma_{\min}(A)$ denotes the minimum eigenvalue of the matrix $A$. $\mathbb{B}_d^1$ denotes the unit ball in $\R^d$ and $\mathbb{S}_{d}^+$ denotes the set of all $d\times d$ positive definite matrices. For functions $f,g:\cX \to \R$, $(f-g)(x):=f(x)-g(x)$ and $(f-g)^2(x):=\left(f(x)-g(x)\right)^2$ for any $x \in \cX$. For any $P:\cZ \to \Delta(\cX)$, we denote $(Pf)(z):=\int_{\cX}f(x)P(x|z)dx$ for any $z \in \cZ$, where $\Delta(\cX)$ denotes the set of signed distributions over $\cX$.

\paragraph{Regret Minimization in Episodic MDPs:} An episodic MDP is denoted by $\cM(\mathcal{S},\mathcal{A},H, P^*, r)$, where $\mathcal{S}$ is the state space, $\mathcal{A}$ is the action space (both possibly infinite), $H$ is the length of each episode, $P^*:\cS \times \cA \to \Delta(\cS)$ is an (unknown) transition kernel (a function mapping state-action pairs to signed distribution over the state space) and $r: \mathcal{S}\times \mathcal{A} \to [0,1]$ is a (known) reward function. In episodic MDPs, a (deterministic) policy $\pi$ is given by a collection of $H$ functions $(\pi_1,\ldots,\pi_H)$, where each $\pi_h:\cS \to \cA$ maps a state $s$ to an action $a$. In each episode, an initial state $s_1$ is first picked by the environment (assumed to be fixed and history independent). Then, at each step $h \in [H]$, the agent observes the state $s_h$, picks an action $a_h$ according to $\pi_h$, receives a reward $r(s_h,a_h)$, and then transitions to the next state $s_{h+1}$, which is drawn from the conditional distribution $P^*(\cdot| s_h,a_h)$. The episode ends when the terminal state $s_{H+1}$ is reached. For each state-action pair $(s,a)\in \mathcal{S} \times \mathcal{A}$ and step $h \in [H]$, we define action values $Q^{\pi}_h(s,a)$ and and state values $V_h^{\pi}(s)$ corresponding to a policy $\pi$ as \begin{align*}
    Q^{\pi}_h(s,a)\!=\!r(s,a) \!+\! \mathbb{E}\!\left[\!\sum\nolimits_{h'=h+1}^H \!\!\!\!r(s_{h'},\! \pi_{h'}(s_{h'}\!)\!)\!\!\mid\!\! s_h \!=\! s, a_h \!=\! a\!\right]\!,\quad
    V^{\pi}_h(s)\!=\! Q^{\pi}_h\big(s,\pi_h(s)\big)~,
\end{align*}
where the expectation is with respect to the randomness of the transition distribution $P^*$. It is not hard to see that $Q_h^\pi$ and $V_h^\pi$ satisfy the Bellman equations: 
\begin{align*}
    Q^{\pi}_h(s,a) = r(s,a) + (P^* V^{\pi}_{h+1})(s,a)\,,\;\; \forall h \in [H],\quad \text{with $V_{H+1}^\pi(s)=0$ for all $s \in \cS$.}
\end{align*}


A policy $\pi^*$ is said to be optimal if it maximizes the value for all states $s$ and step $h$ simultaneously, and the corresponding optimal value function is denoted by $V^*_{h}(s)=\sup_{\pi }V^{\pi}_{h}(s)$ for all $h \in [H]$, where the supremum is over all (non-stationary) policies. The agent interacts with the environment for $K$ episodes to learn the unknown transition kernel $P^*$ and thus, in turn, the optimal policy $\pi^*$. At each episode $k \ge 1$, the agent chooses a policy $\pi^k := (\pi^k_1,\ldots,\pi^k_H)$ and a trajectory $(s_h^k,a_h^k,r(s_h^k,a_h^k),s_{h+1}^k)_{h\in [H]}$ is generated. The performance of the learning agent is measured by the cumulative (pseudo) regret accumulated
over $K$ episodes, defined as
\begin{align*}
    R(T) := \sum\nolimits_{k=1}^K\left[ V_1^{*}(s_1^k)-V_1^{\pi^k}(s_1^k)\right],
\end{align*}
where $T=KH$ is total steps in $K$ episodes.

\section{GENERAL MDPs}
\label{sec:gen}
 In this section, we consider general MDPs without any structural assumption on the unknown transition kernel $P^*$.
 In the standard setting \citep{ayoub2020model}, it is assumed that $P^*$ belongs to a known family of transition models $\cP$. Here, in contrast to the standard setting, we do not have the knowledge of $\cP$. Instead, we are given $M$ nested families of transition kernels $\cP_1 \subset \cP_2 \subset \ldots \subset \cP_M$. The smallest such family where the true transition kernel $P^*$ lies
is denoted by $\cP_{m^*}$, where $m^* \in [M]$. However, we do not know the index $m^*$, and our goal is to propose adaptive
algorithms such that the regret depends on the complexity of the family $\cP_{m^*}$. 
In order to achieve this, we need a separability condition on the nested model classes. 



\begin{assumption}[Separability]\label{ass:sep}
There exists a constant $\Delta > 0$ such that for any bounded function $V:\cS \to [0,H]$, transition kernel $P \in \cP_{m^*-1}$, and state-action pair $(s,a)\in \cS \times \cA$, we have
\begin{equation*}
\left((PV)(s,a) - (P^*V)(s,a) \right)^2 \ge \Delta
\end{equation*}
\end{assumption}
This assumption ensures that given a state-action pair, expected values under the true model is well-separated from expected values under any model from non-realizable classes. We need this to hold uniformly over all states and actions since we do not assume any additional structure over state-action spaces. Note that separability is standard and assumptions of similar nature appear in a wide range of model selection problems, specially in the setting of contextual bandits \citep{ghosh2021modelgen,krishnamurthy2021optimal}. Separability condition is also quite standard in statistics, specifically in the area
of clustering and latent variable modelling \citep{balakrishnan2017statistical,mixture-many,ghosh2019max}.

Note that the assumption breaks down for any constant function $V$. However, we will be invoking this assumption with the value functions computed by the learning algorithm (see~\eqref{eq:bellman-rec}). For reward functions that \emph{vary sufficiently} with states and actions, and transition kernels that admit densities, the chance of getting hit by constant value functions is admissibly low. In case the rewards are constant, every policy would anyway incur zero regret rendering the learning problem trivial. The value functions appear in the separability assumption in the first place since we are interested in minimizing the regret. Instead, if one cares only about learning the true model, then separability of transition kernels under some suitable notion of distance (e.g., the KL-divergence) might suffice. Note that in \citet{ghosh2021modelgen,krishnamurthy2021optimal}, the regret is defined in terms of the regression function and hence the separability is assumed on the regression function itself. Model selection without separability is kept as an interesting future work.


\subsection{Algorithm: Adaptive Reinforcement Learning - General (\texttt{ARL-GEN})}

In this section, we provide a novel model selection algorithm \texttt{ARL-GEN} (Algorithm~\ref{algo:generic}) that use successive refinements
over epochs. We use {{\ttfamily UCRL-VTR}} algorithm of \cite{ayoub2020model} as our base algorithm, and add a model
selection module at the beginning of each epoch. In other words, over multiple epochs, we
successively refine our estimates of the proper model class where the true transition kernel $P^*$ lies.

\paragraph{The Base Algorithm:} {{\ttfamily UCRL-VTR}}, in its general form, takes a family of transition models $\cP$ and a confidence level $\delta \in (0,1]$ as its input. At each episode $k$, it maintains a (high-probability) confidence set $\cB_{k-1} \subset \cP$ for the unknown model $P^*$ and use it for optimistic planning. First, it finds the transition kernel
    $P_k = \argmax_{P \in \cB_{k-1}} V^*_{P,1}(s_1^k)$,
   where $V^*_{P,h}$ denote the optimal value function of an MDP with transition kernel $P$ at step $h$.
\texttt{UCRL-VTR} then computes, at each step $h$, the optimal value function $V_h^k:=V^*_{P_k,h}$ under the kernel $P_k$ using dynamic programming. Specifically, starting with $V_{H+1}^k(s,a)=0$ for all pairs $(s,a)$, it defines for all steps $h=H$ down to $1$, 
\begin{equation} \label{eq:bellman-rec}
\begin{split}
 Q_h^k(s,a)  = r(s,a) + (P_kV_{h+1}^k)(s,a),\quad
        V_h^k(s) = \max\nolimits_{a \in \cA} Q_h^k(s,a).
        \end{split}
\end{equation}
Then, at each step $h$, {{\ttfamily UCRL-VTR}} takes the action that maximizes the $Q$-function estimate, i,e. it chooses $a_h^k = \argmax_{a \in \cA} Q_h^k(s_h^k,a)$. 
Now, the confidence set is updated using all the data gathered in the episode. First, {{\ttfamily UCRL-VTR}} computes an estimate of $P^*$ by employing a non-linear value-targeted regression model with data $\big(s_h^j,a_h^j,V_{h+1}^j(s_{h+1}^j)\big)_{j \in [k],h \in [H]}$. 
Note that
$\E[V_{h+1}^k(s_{h+1}^k)|\cG_{h-1}^k] = (P^* V_{h+1}^k)(s_h^k,a_h^k)$,
where $\cG_{h-1}^k$ denotes the $\sigma$-field summarizing the information available just before $s_{h+1}^k$ is observed. This naturally leads to the estimate $\widehat P_{k} = \argmin_{P \in \cP}\cL_k(P)$, where
\begin{align}\label{eq:gen-est}
     \!\!\cL_k(P)\!:=\!\!\sum\nolimits_{j=1}^{k}\!\sum\nolimits_{h=1}^{H}\! \left(\!V_{h+1}^j(s_{h+1}^j)\!-\!(PV_{h+1}^j)(s_h^j,a_h^j) \!\right)^2.
\end{align}
The confidence set $\cB_k$ is then updated by enumerating the set of all transition kernels $P \in \cP$ satisfying $\sum_{j=1}^{k}\sum_{h=1}^{H} \!\left(\!(PV_{h+1}^j)(s_h^j,a_h^j)\!-\!(\widehat P_{k}V_{h+1}^j)(s_h^j,a_h^j) \!\right)^2 \!\!\le\! \beta_{k}(\delta)$ with the confidence width being defined as $\beta_k(\delta)\!:=\!8H^2\!\log\!\left(\!\frac{2\cN\!\left(\!\cP,\frac{1}{k H},\norms{\cdot}_{\infty,1}\!\right)\!}{\delta}\!\right)\! +\!  4 H^2 \!\left(\!2\!+\!\!\sqrt{\!2 \log\! \left(\!\frac{4kH(kH+1)}{\delta} \!\right)\!} \!\right)$, where $\cN(\cP,\cdot,\cdot)$ denotes the covering number of the family $\cP$.
\footnote{For any $\alpha > 0$, $\cP^\alpha$ is an $(\alpha,\norms{\cdot}_{\infty,1})$ cover of $\cP$ if for any $P \in \cP$ there exists an $P'$ in $\cP^\alpha$ such that $\norms{P' - P}_{\infty,1}:=\sup_{s,a}\int_{\cS}|P'(s'|s,a)-P(s'|s,a)|ds' \le \alpha$.} Then, one can show that $P^*$ lies in the confidence set $\cB_k$ in all episodes $k$ with probability at least $1-\delta$. Here, we consider a slight different expression of $\beta_k(\delta)$ as compared to \cite{ayoub2020model}, but the proof essentially follows the same technique. Please refer to Appendix~\ref{app:gen} for further details, and for all the proofs of this section.

\paragraph{Our Approach:} We consider doubling epochs - at each epoch $i \ge 1$, \texttt{UCRL-VTR} is run for $k_i=2^i$ episodes. At the beginning of $i$-th epoch, using all the data of
previous epochs, we add a model selection module as follows. First, we compute, for each family $\cP_m$, the transition kernel $\widehat P_m^{(i)}$, that minimizes the empirical loss $\cL_{\tau_{i-1}}(P)$ over all $P \in \cP_m$ (see \eqref{eq:gen-est}), where $\tau_{i-1}:=\sum_{j=1}^{k-1}k_j$ denotes the total number of episodes completed before epoch $i$. Next, we compute the average empirical loss $T_m^{(i)}:=\frac{1}{\tau_{i-1}H}\cL_{\tau_{i-1}}(\widehat P_m^{(i)})$ for the model $\widehat P_m^{(i)}$. Finally, we compare $T_m^{(i)}$ to a pre-calculated threshold $\gamma_i$, and pick the transition family for which $T_m^{(i)}$ falls
below such threshold (with smallest $m$, see Algorithm~\ref{algo:generic}). After selecting the family, we run \texttt{UCRL-VTR} for this family with confidence level $\delta_i=\frac{\delta}{2^i}$, where $\delta \in (0,1]$ is a parameter of the algorithm.

\begin{algorithm}[t!]
  \caption{Adaptive Reinforcement Learning - General -- \texttt{ARL-GEN}}
  \begin{algorithmic}[1]
 \STATE  \textbf{Input:} Confidence parameter $\delta$,  function classes $\cP_1 \subset \cP_2 \subset \ldots \subset \cP_M$, thresholds $\lbrace\gamma_i\rbrace_{i\ge 1}$
  \FOR{epochs $i=1,2 \ldots$}
  \STATE Set $\tau_{i-1}=\sum_{j=1}^{i-1}k_j$
  \FOR{function classes $m=1,2 \ldots,M$}
  \STATE Compute $\widehat P^{(i)}_m = \argmin\nolimits_{P \in \cP_m}\!\sum\nolimits_{k=1}^{\tau_{i-1}}\!\sum\nolimits_{h=1}^{H} \!\left(V_{h+1}^k(s_{h+1}^k)\!-\!(PV_{h+1}^k)(s_h^k,a_h^k) \right)^2$
  \STATE Compute $T^{(i)}_m = \frac{1}{\tau_{i-1}H}\!\sum\nolimits_{k=1}^{\tau_{i-1}}\!\sum\nolimits_{h=1}^{H} \!\!\left(V_{h+1}^k(s_{h+1}^k)\!-\!(\widehat P^{(i)}_m V_{h+1}^k)(s_h^k,a_h^k) \right)^2$
   \ENDFOR
  \STATE Set $m^{(i)}=\min\lbrace m \in [M]: T_m^{(i)} \le \gamma_i\rbrace$, $k_i=2^i$ and $\delta_i=\delta/2^i$ 
  \STATE Run {{\ttfamily UCRL-VTR}} for the family $\cP_{m^{(i)}}$ for $k_i$ episodes with confidence level $\delta_i$
    \ENDFOR
  \end{algorithmic}
  \label{algo:generic}
\end{algorithm}


\subsection{Analysis of \texttt{ARL-GEN}}

First, we present our main result which states that the model selection procedure of \texttt{ARL-GEN} (Algorithm~\ref{algo:generic}) succeeds with high probability after a certain number of epochs. To this end, we denote by $\mathbb{M}_m=\log(\cN(\cP_{m},1/T,\norms{\cdot}_{\infty,1}))$ the metric entropy (with scale $1/T$) of the family $\cP_{m}$. 


\begin{lemma}[Model selection of \texttt{ARL-GEN}]
\label{lem:gen_infinite}
Fix a $\delta \in (0,1]$ and suppose Assumption~\ref{ass:sep} holds. Suppose the thresholds are set as $\gamma_i = T^{(M)}_m + \frac{\sqrt{i}}{2^{i/2}},$ where $T_M^{(M)}$ is the test statistic for the $M$-th model class, and $i$ is the epoch number. 
Then, with probability at least $1-3M\delta$, \texttt{ARL-GEN} identifies the correct model class $\mathcal{P}_{m^*}$ from epoch $i \geq i^*$, where epoch length of $i^*$ satisfies
\begin{align*}
    2^{i^*} \!\!\!\geq\! C \max\! \left \lbrace\!\! \frac{H^3 \!\log\! K}{\Delta^2} \!\log\!\frac{1}{\delta}, H\left(\!\mathbb{M}_M\!+\! \log\!\frac{1}{\delta}\!\!\right) \right \rbrace
    \end{align*}
for a sufficiently large universal constant $C>1$.
\end{lemma}
\begin{remark}[Dependence on the biggest class]
\label{rem:biggest-class-dependence}
Note that we choose a threshold that depends on the epoch number and the test statistic of the biggest class. Here we crucially exploit the fact that the biggest class always contains the true model class and use this to design the threshold.
\end{remark}
\textit{Proof idea.} In order to do model selection, we first obtain upper bounds on the test statistics $T^{(i)}_m$ for model classes that includes $\mathcal{P}^*$. We accomplish this by carefully defining a martingale difference sequence that depends on the value function estimates $V^k_{h+1}$, and invoking Azuma-Hoeffding inequality. We then obtain a lower bound on $T^{(i)}_m$ for model classes not containing $\mathcal{P}^*$ by leveraging Assumption~\ref{ass:sep} (separability). Combining the above two bounds yields the desired result. 

\paragraph{Regret Bound:} In order to present our regret bound, we define, for each model model class $\cP_m$, a collection of functions $\cF_m := \left\lbrace f:\cS \times \cA \times \cV_m \to \R\right\rbrace$ such that any $f \in \cF_m$ satisfies $f(s,a,V) = (PV) (s,a)$ for some $P \in \cP_m$,
where $\cV_m:=\lbrace V^*_{P,h}:P \in \cP_m ,h \in [H]\rbrace$ denotes the set of optimal value functions under the transition family $\cP_m$.
By one-to-one correspondence, we have $\cF_1 \subset \cF_2 \subset \ldots \subset \cF_M$, and
the complexities of these function classes determine the learning complexity of the RL problem under consideration. We characterize the complexity of each function class $\cF_m$ by its \emph{eluder dimension}, which is defined as follows.
\begin{definition}[Eluder dimension]\label{def:eluder}
The $\varepsilon$-eluder dimension $\dimn_{\cE}(\cF_m,\varepsilon)$ of the function class $\cF_m$ is the length of the longest sequence $\lbrace(s_i,a_i,V_i)\rbrace_{i=1}^{n} \subseteq \cS \times \cA \times \cV_m$ of state-action-optimal value function tuples under the transition family $\cP_m$ such that for some $\varepsilon' \ge \varepsilon$ and for each $ i \in\lbrace 2,\ldots,n\rbrace$, we have $\sup\nolimits_{f_1,f_2 \in \cF_m} \big(f_1-f_2\big)(s_i,a_i,V_i) > \varepsilon'~,$ given that $\sqrt{\sum\nolimits_{j=1}^{i-1}(f_1-f_2)^2(s_i,a_i,V_i)} \leq \varepsilon'$.
\end{definition}
The notion of eluder dimension was first
introduced by \cite{russo2013eluder} to characterize the complexity of function classes, and since then it has been widely used \citep{osband2014model,ayoub2020model,wang2020provably}.
We denote by $d_{\cE}^*=\dimn_{\cE}(\cF_{m^*},1/T)$ the $(1/T)$-eluder dimension of the function class $\cF_{m^*}$ corresponding to the (realizable) family $\cP_{m^*}$. Furthermore, we denote by $\mathbb{M}^*:=\mathbb{M}_{m^*}$ the metric entropy of the family $\cP_{m^*}$.
Then, armed with Lemma~\ref{lem:gen_infinite}, we obtain the following regret bound.\footnote{For ease
of representation, we assume the transition kernel $P^*$  to be fixed for all steps $h$. Our results extends naturally to the setting, where there are $H$ different kernels. This would only add a multiplicative $\sqrt{H}$ factor in the regret bound \citep{jin2019provably}. Moreover, our results can be extended to the setting where the rewards are also unknown.}
\begin{theorem}[Cumulative regret of \texttt{ARL-GEN}]
\label{thm:general}
Suppose the conditions of Lemma~\ref{lem:gen_infinite} hold. Then, for any $\delta \!\in\! (0,1]$, running \texttt{ARL-GEN} for $K$ episodes yields a regret bound
 \begin{align*}
     R(T)  \!\leq\! \mathcal{O}\!\! \left(\!\! \max\! \left \lbrace\!\! \frac{\!H^3 \log\! K}{\Delta^2} \!\log\!\frac{1}{\delta},\! H\!\!\left(\!\!\mathbb{M}_M\!\!+\! \log\!\frac{1}{\delta}\!\!\right) \!\!\right \rbrace\!\! \right)\! \!+\! \mathcal{O}\!\!\left(\!\!H^2 d_{\cE}^* \log\! K 
      \!\!+\! H\! \sqrt{\!Td_{\cE}^*\!\!\left(\!\!\mathbb{M}^*\!\!+\!\log\!\frac{1}{\delta}\!\right)\!\! \log\! K\log\!\frac{T}{\delta}\!}   \!\right)\!
 \end{align*}
 with probability at least $1- 3 M\delta-2\delta$. \footnote{One can choose $\delta = 1/\poly(M)$ to obtain a high-probability bound which only adds an extra $\log M$ factor.}
\end{theorem}
The first term in the regret bound captures the cost of model selection -- the cost suffered before accumulating enough samples to infer the correct model class (with high probability). It has weak (logarithmic) dependence on the number of episodes $K$ and hence considered as a minor term, in the setting where $K$ is large. Hence, model selection is essentially \emph{free} upto log factors. Let us now have a close look at this term. It depends on the metric entropy of the biggest model class $\mathcal{P}_M$. This stems from the fact that 
the thresholds $\lbrace\gamma_i\rbrace_{i \ge 1}$ depends on the test statistic of $\cP_M$ (see Remark~\ref{rem:biggest-class-dependence}). We believe that, without additional assumptions, one can't get rid of this (minor) dependence on the complexity of the biggest class.

The second term is the major one ($\sqrt{T}$ dependence on total number of steps), which essentially is the cost of learning the true kernel $P^*$. Since in this phase, we  basically run \texttt{UCRL-VTR} for the correct model class, our regret guarantee matches to that of an oracle with the apriori knowledge of the correct class. Note that if we simply run a non model-adaptive algorithm (e.g. \texttt{UCRL-VTR}) for this problem, the regret would be $\widetilde{\mathcal{O}}(H\sqrt{T d_{\cE,M} \mathbb{M}_M})$, where $d_{\cE,M}$ denotes the eluder dimension of the largest model class $\mathcal{P}_M$. In contrast, by successively testing and thresholding, our algorithm adapts to the complexity of the smallest function class containing the true model class. 

\begin{remark}[Dependence on $\Delta$] Dependence on the separation $\Delta$ is reflected in the minor term of the regret bound. If the separation is small, it is difficult for \texttt{ARL-GEN} to separate out the model classes. Hence, it requires additional exploration, and as a result the regret increases. It is worth noting that \texttt{ARL-GEN} does not require any knowledge of $\Delta$. Rather, it \emph{adapts} to the separation present in the problem. Another 
interesting fact of Theorem~\ref{thm:general} is that it does not require any minimum separation across model classes. This is in sharp contrast with existing results in statistics (see, e.g. \cite{balakrishnan2017statistical,mixture-many}). Even if $\Delta$ is quite small, Theorem~\ref{thm:general} gives a model selection guarantee. Now, the cost of separation appears anyways in the minor term, and hence in the long run, it does not effect the overall performance of the algorithm.

\end{remark}

\section{LINEAR
KERNEL MDPs}
\label{sec:lin}

In this section, we consider a special class of MDPs called \emph{linear kernel MDPs} \citep{jia2020model}. Roughly speaking, it means that the transition kernel $P^*$ can be represented as a linear function of a given feature map $\phi:\mathcal{S}\times \mathcal{A}\times \mathcal{S} \to \mathbb{R}^d$. Formally, we have the following definition. 

\begin{definition}[Linear kernel MDP]
An MDP $\cM\left(\mathcal{S},\mathcal{A},H, P^*, r\right)$ is called a $b$-bounded linear kernel MDP if there exists a \emph{known} feature mapping $\phi:\mathcal{S}\times \mathcal{A}\times \mathcal{S} \to \mathbb{R}^d$ and an \emph{unknown} vector $\theta^* \in \mathbb{R}^d$ with $\eucnorm{\theta^*} \le b$ such that $P^*(s'| s,a) = \inprod{\phi(s,a,s')}{\theta^*}$ for any $(s,a,s') \in \mathcal{S}\times \mathcal{A}\times\mathcal{S}$.
\end{definition} 

The MDP is parameterized by the unknown parameter $\theta^*$, and a natural measure of complexity of the problem is the dimension or sparsity (number of non-zero coordinates) of $\theta^*$. We propose an adaptive algorithm \texttt{ARL-LIN(dim)} that tailors to the sparsity $\|\theta^*\|_0 $ of $\theta^*$ (Algorithm~\ref{algo:dim}). 

\begin{algorithm}[t!]
  \caption{Adaptive Reinforcement Learning - Linear -- \texttt{ARL-LIN(dim)}}
  \begin{algorithmic}[1]
 \STATE \textbf{Input:} Initial phase length $k_0$, confidence level $\delta \in (0,1]$, norm upper bound $b$
 \STATE Initialize estimate of $\theta^*$ as $\widehat{\theta}^{(0)} = \mathbf{1}$
 \FOR{epochs $i=0,1,2 \ldots$}
 \STATE Set $k_i = 36^{i} k_0$ and  $\delta_i = \delta/2^{i}$
 \STATE Refine estimate of non-zero coordinates: $\cD^{(i)}:= \{i : |\widehat{\theta}^{(i)}| \geq (0.5)^{i+1} \}$
 \STATE Play \texttt{UCRL-VTR-LIN} only restricted to $\mathcal{D}^{(i)}$ co-ordinates for $k_i$ episodes with norm upper bound $b$ and confidence level $\delta_i$
 \STATE Play {{\ttfamily UCRL-VTR-LIN}} in full $d$ dimension for $6^i\lceil \sqrt{k_0} \rceil$ episodes starting from where we left of in epoch $i-1$ with norm upper bound $b$ and confidence level $\delta$ 
 \STATE Compute an estimate of $\theta^*$ as $\widehat{\theta}^{(i+1)}=\widehat{\theta}_{\tau_i}$ using \eqref{eq:theta-est}, where $\tau_i= \sum_{j=0}^{i}6^j\lceil \sqrt{k_0} \rceil$
    \ENDFOR
  \end{algorithmic}
  \label{algo:dim}
\end{algorithm}


\paragraph{The Base Algorithm:} We take the algorithm of \cite{jia2020model} as our base algorithm, which is an adaptation of \texttt{UCRL-VTR} for linear kernel MDPs (henceforth, denoted as \texttt{UCRL-VTR-LIN}).
\texttt{UCRL-VTR-LIN} takes the upper bound $b$ of $\eucnorm{\theta^*}$ and a confidence level $\delta \in (0,1]$ as its input.
Now, let us have a look at how \texttt{UCRL-VTR-LIN} constructs the confidence ellipsoid at the $k$-th episode. 
Observe that at step $h$ of episode $k$, we have $(P^*V_{h+1}^k)(s_h^k,a_h^k)=\inprod{\phi_{V_{h+1}^k}(s_h^k,a_h^k)}{\theta^*}$,
where $V_{h+1}^k$ is a an estimate of the value function constructed using all the data received before episode $k$.
An estimate $\widehat \theta_k$ of $\theta^*$ is then computed by solving the following optimization problem
\begin{align}
   \!\!\min_{\theta \in \R^d}\! \sum\nolimits_{j=1}^{k}\!\sum\nolimits_{h=1}^{H}\!\!\left(\!V_{h+1}^j(s_{h+1}^j\!)\!-\!\inprod{\phi_{V_{h+1}^j}\!\!\!(s_h^j,a_h^j)}{\!\theta}  \!\right)^2\! \!\!+\! \eucnorm{\theta}^2.
    \label{eq:theta-est}
\end{align}
The confidence ellipsoid is then constructed as $\cB_k = \left\lbrace \theta \in \R^d : \norms{\Sigma_k^{1/2}(\theta-\widehat\theta_k)}^2 \le \beta_k(\delta)\right\rbrace$, where $\Sigma_{k} = I + \sum_{j=1}^{k}\sum_{h=1}^{H}  \phi_{V_{h+1}^j}(s^j_h,a^j_h) \phi_{V_{h+1}^j} (s^j_h,a^j_h)^\top$ and $\beta_{k}(\delta) = O\left((b^2 + H^2d\log(kH)\log^2(k^2H/\delta)\right)$. Hence, the $Q$-function estimates (see~\eqref{eq:bellman-rec}) of \texttt{UCRL-VTR-LIN} at the $k$-th episode take the form
\begin{align*}
    Q_h^k(s,a) = r(s,a) + \inprod{\phi_{V_{h+1}^k}(s,a)}{\widehat\theta_{k-1}} + \sqrt{\beta_{k-1}(\delta)}\norms{\Sigma_{k-1}^{-1/2}\phi_{V_{h+1}^k}(s,a)}.
\end{align*}
Using these, \texttt{UCRL-VTR-LIN} defines the value estimates as $V_h^k(s) = \min\lbrace\max_{a \in \cA} Q_h^k(s,a),H\rbrace$ to keep those bounded. Then, \citet{jia2020model} show that $\theta^*$ lies in the confidence ellipsoid $\cB_{k}$ in all episodes $k$ with probability at least $1-\delta$.

\paragraph{Our approach:} The proposed algorithm works over multiple epochs, and we use diminishing thresholds to estimate the support of $\theta^*$. The algorithm is parameterized by the initial phase length $k_0$, and the confidence level $\delta \in (0,1]$. {\ttfamily ARL-LIN(dim)} proceeds in epochs numbered $0,1,\ldots$, increasing with time. 
Each epoch $i$ is divided into two phases - {\em (i)} a regret minimization phase lasting $36^i k_0$ episodes, {\em (ii)} followed by a support estimation phase lasting $6^i \lceil\sqrt{k_0}\rceil$ episodes. Thus, each epoch $i$ lasts for a total of $36^i k_0 + 6^i \lceil \sqrt{k_0} \rceil$ episodes. 
At the beginning of epoch $i \geq 0$, $\mathcal{D}^{(i)} \subseteq [d]$ denotes the set of `active coordinates' -- the estimate of non-zero coordinates of $\theta^*$.



In the regret minimization phase of epoch $i$, a fresh instance of {{\ttfamily UCRL-VTR-LIN}} is spawned, with the dimensions restricted only to the set $\mathcal{D}^{(i)}$ and confidence level $\delta_i:= \frac{\delta}{2^i}$.
On the other hand, in the support estimation phase, we continue running the {{\ttfamily UCRL-VTR-LIN}} algorithm in full $d$ dimension, from the point where we left of in epoch $i-1$. Concretely, one should think of the support estimation phases over epochs as a single run of {{\ttfamily UCRL-VTR-LIN}} in the full $d$ dimension with confidence level $\delta$ and norm upper bound $b$. At the end of each epoch $i\geq 0$, let $\tau_i:= \sum_{j=0}^{i}6^j\lceil \sqrt{k_0} \rceil$ denote the total number episodes run in the support estimation phases. Then, {\ttfamily ARL-LIN(dim)} forms an estimate of $\theta^*$ as $\widehat{\theta}^{(i+1)}:=\widehat{\theta}_{\tau_i}$, where, for any $k \ge 1$, $\widehat\theta_k$ is as defined in \eqref{eq:theta-est}.
The active coordinate set $\mathcal{D}^{(i+1)}$ for the next epoch is then the coordinates of $\widehat{\theta}^{(i+1)}$ with  magnitude exceeding $(0.5)^{i+1}$. By this careful choice of exploration periods and thresholds, we show that the estimated support of $\theta^*$ is equal to the true support, for all but finitely many epochs. Thus, after a finite number of epochs, the true support of $\theta^*$ \emph{locks-in}, and thereafter the agent incurs the regret that an oracle knowing the true support would incur. Hence, the extra regret we incur with respect to an the oracle (which knows the support of $\theta^*$) is small.

Now, we informally state the regret bound of \texttt{ARL-LIN(dim)}. The formal statement is deferred to Appendix~\ref{app:dim}, Theorem~\ref{thm:dim}. We let $d^*=\|\theta^*\|_0$ to denote the sparsity of $\theta^*$.

\begin{proposition}[Informal regret bound of \texttt{ARL-LIN(dim)}]
Under a technical assumption on the feature mapping $\phi$, if \texttt{ARL-LIN(dim)} is run for $K$ episodes with suitably chosen initial phase length $k_0$, its regret satisfies (with high probability) the following:
\begin{align*}
    R(T) \!=\! \Tilde{\mathcal{O}}\!\left(\! \frac{Hk_0}{\gamma^{5.18}}    \!+\! \!\left(\!bd^*\!\sqrt{\!H^3T} \!+\! b\,dH^2K^{1/4}\!\right)\polylog \,T\!\!\right),
\end{align*}
where $\gamma =  \min\{|\theta^*(j)|: \theta^*(j) \neq 0\} $ with $\theta^*(j)$ denoting the $j$-th coordinate of $\theta^*$.
\end{proposition}

Note that the cost of model selection is $\Tilde{\mathcal{O}}\left(\frac{H}{\gamma^{5.18}} k_0 + b\,dH^2K^{1/4}\right)$, which has a weaker dependence on the number of episodes $K$ than the leading term, and hence has a minor influence on the performance of the algorithm in the long run.


\begin{remark}[Dependence on $\gamma$] The regret bound depends on $\gamma$ -- the minimum absolute non-zero entry of $\theta^*$, and hence is instance-dependent. At this point, it is worth mentioning that we haven't optimized over the choice of epoch lengths and dependence on $\gamma$. In particular, choosing the support estimation period as $2^i \lceil \sqrt{k_0} \rceil$, the regret minimization period as $4^i k_0$ and the threshold as $(0.9)^i$, one can ensure the support locks-in after only $2$ epochs. However, the dependence on $\gamma$ in this case becomes worse. Therefore, the support estimation period, regret minimization period and threshold selection can be kept as tuning parameters.
\end{remark}

\paragraph{Norm as complexity measure:}
In Appendix~\ref{app:norm}, we take the norm, $\norms{\theta^*}$, of the unknown parameter $\theta^*$ as a measure of complexity of linear kernel MDPs. We develop an algorithm
\texttt{ARL-LIN(norm)} which adapts to $\norms{\theta^*}$ (see Algorithm~\ref{algo:norm}). This is in sharp contrast to non-adaptive algorithms (e.g. \texttt{UCRL-VTR-LIN} of \citet{jia2020model}) that requires an upper bound on the norm of $\theta^*$. In this setting also, we show that model selection is (order-wise) free. Please refer to Theorem~\ref{thm:norm} for the exact regret bound of \texttt{ARL-LIN(norm)}.






\newpage
\bibliography{model_rl,2018library,Bandit_RL_bib}

 \newpage
 \appendix
 
\begin{center}
{\huge Appendix}
\end{center}

\section{DETAILS FOR SECTION \ref{sec:gen}}
\label{app:gen}
\subsection{Confidence Sets in \texttt{UCRL-VTR}}
We first describe how the confidence sets are constructed in \texttt{UCRL-VTR}. Note that the procedure is similar to that done in \cite{ayoub2020model}, but with a slight difference. Specifically, we define the confidence width as a function of complexity of the transition family $\cP$ on which $P^*$ lies; rather than the complexity of a value-dependent function class induced by $\cP$ (as done in \cite{ayoub2020model}). We emphasize that this small change makes the model selection procedure easier to understand without any effect on the regret.

Let us define, for any two transition kernels $P,P' \in \cP$ and any episode $k \ge 1$, the following
\begin{align*}
    \cL_k(P) &= \sum_{j=1}^{k}\sum_{h=1}^{H} \left(V_{h+1}^j(s_{h+1}^j)-(PV_{h+1}^j)(s_h^j,a_h^j) \right)^2,\\ \cL_k(P,P') &= \sum_{j=1}^{k}\sum_{h=1}^{H} \left((PV_{h+1}^j)(s_h^j,a_h^j)-( P'V_{h+1}^j)(s_h^j,a_h^j) \right)^2.
\end{align*}
Then, the confidence set at the end of episode $k$ is constructed as 
\begin{align*}
    \cB_k = \left\lbrace P \in \cP \,\big |\, \cL_k(P,\widehat P_k) \le \beta_{k}(\delta)\right \rbrace,
\end{align*}
where $\widehat P_k =\argmin_{P \in \cP}\cL_k(\cP)$ denotes an estimate of $P^*$ after $k$ episodes. The confidence width $\beta_k(\delta)\equiv \beta_{k}(\cP,\delta)$ is set as
\begin{align*}
    \beta_k(\delta):=
     \begin{cases}
     8H^2\log\left(\frac{|\cP|}{\delta}\right) & \text{if}\; \cP\; \text{is finite,}\\
     8H^2\log\left(\frac{2\cN\left(\cP,\frac{1}{k H},\norms{\cdot}_{\infty,1}\right)}{\delta}\right) +  4 H^2 \left(2+\sqrt{2 \log \left(\frac{4kH(kH+1)}{\delta} \right)} \right) & \text{if}\; \cP \; \text{is infinite.}
    \end{cases}
\end{align*}

\begin{lemma}[Concentration of $P^*$]\label{lem:conc-ucrlvtr}
For any $\delta \in (0,1]$, with probability at least $1-\delta$, uniformly over all episodes $k \ge 1$, we have $P^* \in \cB_k$.
\end{lemma}
\begin{proof}
First, we define, for any fixed $P \in \cP$ and $(h,k) \in [H] \times [K]$, the quantity
\begin{align*}
    Z_h^{k,P} := 2 \left(( P^* V_{h+1}^k)(s_h^k,a_h^k)-(P V_{h+1}^k)(s_h^k,a_h^k) \right)\left( V_{h+1}^k(s_{h+1}^k)-( P^* V_{h+1}^k)(s_h^k,a_h^k)\right).
\end{align*}
Then, we have
\begin{align}
    \cL_k(\widehat P_k) = \cL_k(P^*) +\cL_k(P^*,\widehat P_{k}) + \sum_{j=1}^{k}\sum_{h=1}^{H}Z_h^{j,\widehat P_k}~.
    \label{eq:break-up}
\end{align}
Using the notation $y_h^k=V_{h+1}^k(s_{h+1}^k)$, we can rewrite $Z_h^{k,P}$ as
\begin{align*}
    Z_h^{k,P} := 2 \left(( P^* V_{h+1}^k)(s_h^k,a_h^k)-(P V_{h+1}^k)(s_h^k,a_h^k) \right)\left( y_h^k-\E[y_h^k|\cG_{h-1}^k]\right),
\end{align*}
where $\cG_{h-1}^k$ denotes the $\sigma$-field summarising all the information available just before $s_{h+1}^k$ is observed.
Note that $Z_h^{k,P}$ is $\cG_h^k$-measurable,  Moreover, since $V_{h+1}^k \in [0,H]$, $Z_h^{k,P}$ is $2H |( P^* V_{h+1}^k)(s_h^k,a_h^k)-(P V_{h+1}^k)(s_h^k,a_h^k)|$-sub-Gaussian conditioned on $\cG_{h-1}^k$. Therefore, for any $\lambda < 0$, with probability at least $1-\delta$, we have
\begin{align*}
    \forall k \ge 1,\quad\sum_{j=1}^{k}\sum_{h=1}^{H} Z_h^{j,P} \ge \frac{1}{\lambda}\log(1/\delta)+ \frac{\lambda}{2}\cdot 4H^2 \sum_{j=1}^{k}\sum_{h=1}^{H}\left(( P^* V_{h+1}^j)(s_h^j,a_h^j)-(P V_{h+1}^j)(s_h^j,a_h^j) \right)^2~.
\end{align*}
Setting $\lambda = -1/(4H^2)$, we obtain for any fixed $P \in \cP$, the following:
\begin{align}
    \forall k \ge 1,\quad\sum_{j=1}^{k}\sum_{h=1}^{H} Z_h^{j,P} \ge -4H^2\log\left(\frac{1}{\delta}\right)- \frac{1}{2} \cL_k(P^*,P)~.
    \label{eq:subg-conc}
\end{align}
with probability at least $1-\delta$. We consider both the cases -- when $\cP$ is finite and when $\cP$ is infinite.

\paragraph{Case 1 -- finite $\cP$:}
We take a union bound over all $P \in \cP$ in \eqref{eq:subg-conc} to obtain that
\begin{align}
    \forall k \ge 1,\;\; \forall P \in \cP,\quad\sum_{j=1}^{k}\sum_{h=1}^{H} Z_h^{j,P} \ge -4H^2\log\left(\frac{|\cP|}{\delta}\right)- \frac{1}{2} \cL_k(P^*,P)
    \label{eq:finite}
\end{align}
with probability at least $1-\delta$. By construction, $\widehat P_k \in \cP$ and $\cL_k(\widehat P_k) \le \cL_k(P^*)$. Therefore, from \eqref{eq:break-up}, we have
\begin{align*}
   \forall k \ge 1,\;\;  \cL_k(P^*,\widehat P_k) \le 8H^2\log\left(\frac{|\cP|}{\delta}\right)
\end{align*}
with probability at least $1-\delta$, which proves the result for finite $\cP$.

\paragraph{Case 2 -- infinite $\cP$:}

Fix some $\alpha > 0$. Let $\cP^\alpha$ denotes an $(\alpha,\norms{\cdot}_{\infty,1})$ cover of $\cP$, i.e., for any $P \in \cP$, there exists an $P^\alpha$ in $\cP^\alpha$ such that $\norms{P^\alpha - P}_{\infty,1}:=\sup_{s,a}\int_{\cS}|P^\alpha(s'|s,a)-P(s'|s,a)|ds' \le \alpha$. Now, we take a union bound over all $P^\alpha \in \cP^\alpha$ in \eqref{eq:subg-conc} to obtain that
\begin{align*}
    \forall k \ge 1,\;\; \forall P^\alpha \in \cP^\alpha,\quad\sum_{j=1}^{k}\sum_{h=1}^{H} Z_h^{j,P^\alpha} \ge -4H^2\log\left(\frac{|\cP^\alpha|}{\delta}\right)- \frac{1}{2} \cL_k(P^*,P^\alpha)~.
\end{align*}
with probability at least $1-\delta$, and thus, in turn, 
\begin{align}
    \forall k \ge 1,\;\; \forall P\in \cP,\quad\sum_{j=1}^{k}\sum_{h=1}^{H} Z_h^{j,P} \ge -4H^2\log\left(\frac{|\cP^\alpha|}{\delta}\right)- \frac{1}{2} \cL_k(P^*,P)+\zeta_k^{\alpha}(P).
    \label{eq:comb-one-conf}
\end{align}
with probability at least $1-\delta$,
where $\zeta_k^{\alpha}(P)$ denotes the discretization error:
\begin{align*}
    &\zeta_k^{\alpha}(P)\\ =& \sum_{j=1}^{k}\sum_{h=1}^{H} \left(Z_h^{j,P}-  Z_h^{j,P^\alpha}\right) + \frac{1}{2}\cL_k(P^*,P) - \frac{1}{2}\cL_k(P^*,P^\alpha)\\
    =& \sum_{j=1}^{k}\sum_{h=1}^{H} \left(2y_h^j \left(( P^\alpha V_{h+1}^j)(s_h^j,a_h^j)-(P V_{h+1}^j)(s_h^j,a_h^j) \right) + \frac{1}{2}(P V_{h+1}^j)^2(s_h^j,a_h^j) - \frac{1}{2}(P^\alpha V_{h+1}^j)^2(s_h^j,a_h^j)\right).
\end{align*}
Since $\norms{P-P^\alpha}_{\infty,1} \le \alpha$ and $\norms{V_{h+1}^k}_\infty \le H$, we have 
\begin{align*}
 \left |(P^\alpha V_{h+1}^k)(s,a)-(PV_{h+1}^k)(s,a)\right| \le \alpha H~,   
\end{align*}
which further yields
\begin{align*}
    \left|(P^\alpha V_{h+1}^k)^2(s,a)- (P V_{h+1}^k)^2(s,a)\right| &\le \max_{|\xi| \le \alpha H} \left | \left((PV_{h+1}^k)(s,a)+\xi\right)^2-(PV_{h+1}^k)(s,a)^2 \right|\\ &\le 2 \alpha H^2 +\alpha^2H^2~.
\end{align*}
Therefore, we can upper bound the discretization error as
\begin{align*}
    |\zeta_k^{\alpha}(P)| &\le 2 \alpha H \sum_{j=1}^{k}\sum_{h=1}^{H}|y_h^j| + \sum_{j=1}^{k}\sum_{h=1}^{H} \left(\alpha H^2 + \frac{\alpha^2H^2}{2}\right)\\ &\le 2 \alpha H \sum_{j=1}^{k}\sum_{h=1}^{H}|y_h^j-\E[y_h^j|\cG_{h-1}^j]| + \sum_{j=1}^{k}\sum_{h=1}^{H} \left(3\alpha H^2 + \frac{\alpha^2H^2}{2}\right)~.
\end{align*}
Since $y_h^k-\E[y_h^k|\cG_{h-1}^k]$ is $H$-sub-Gaussian conditioned on $\cG_{h-1}^k$, we have 
\begin{align*}
    \forall k \ge 1,\forall h \in [H],\quad |y_h^k-\E[y_h^k|\cG_{h-1}^k]| \le H\sqrt{2 \log \left(\frac{2kH(kH+1)}{\delta} \right)}
\end{align*}
with probability at least $1-\delta$.
Therefore, with probability at least $1-\delta$, the discretization error is bounded for all episodes $k \ge 1$ as
\begin{align*}
    |\zeta_k^{\alpha}(P)| &\le kH \left( 2\alpha H^2 \sqrt{2 \log \left(\frac{2kH(kH+1)}{\delta} \right)} + 3\alpha H^2+\frac{\alpha^2H^2}{2}\right)\\
    & \le \alpha kH \left( 2 H^2 \sqrt{2 \log \left(\frac{2kH(kH+1)}{\delta} \right)} + 4H^2\right),
\end{align*}
where the last step holds for any $\alpha \leq 1$. Therefore, from \eqref{eq:comb-one-conf}, we have
\begin{align}
   \forall k \ge 1,\;\; \forall P\in \cP,\quad\sum_{j=1}^{k}\sum_{h=1}^{H} Z_h^{j,P} &\ge -4H^2\log\left(\frac{|\cP^\alpha|}{\delta}\right)- \frac{1}{2} \cL_k(P^*,P)\nonumber\\ &  \quad \quad- \alpha kH \left( 2 H^2 \sqrt{2 \log \left(\frac{2kH(kH+1)}{\delta_i} \right)} + 4H^2\right),
   \label{eq:infinite}
\end{align}
with probability at least $1-2\delta$.
Now, setting $\alpha=\frac{1}{k H}$, we obtain, from \eqref{eq:break-up}, that
\begin{align*}
    \forall k \ge 1,\; \cL_k(P^*,\hat P_{k}) \le 8H^2\log\left(\frac{2\cN\left(\cP,\frac{1}{k H},\norms{\cdot}_{\infty,1}\right)}{\delta}\right) +  4 H^2 \left(2+\sqrt{2 \log \left(\frac{4kH(kH+1)}{\delta} \right)} \right)
\end{align*}
with probability at least $1-\delta$, which proves the result for infinite $\cP$.
\end{proof}

\subsection{Model Selection in \texttt{ARL-GEN}}

First, we find concentration bounds on the test statistics $T_m^{(i)}$ for all epochs $i \ge 1$ and class indexes $m \in [M]$, which are crucial to prove the model selection guarantee (Lemma \ref{lem:gen_infinite}) of \texttt{ARL-GEN}.

\paragraph{1. Realizable model classes:} Fix a class index $m \ge m^*$. In this case, the true model $P^* \in \cP_m$. Therefore, the we can upper bound the empirical at epoch $i$ as
\begin{align*}
    T_m^{(i)} \le \frac{1}{\tau_{i-1}H}\sum_{k=1}^{\tau_{i-1}}\sum_{h=1}^{H} \left(V_{h+1}^k(s_{h+1}^k)-( P^* V_{h+1}^k)(s_h^k,a_h^k) \right)^2 = \frac{1}{\tau_{i-1}H}\sum_{k=1}^{\tau_{i-1}}\sum_{h=1}^{H} \left( y_h^k - \E[y_h^k|\cG_{h-1}^k]\right)^2.
\end{align*}
Now, we define the random variable $m_h^k := \left( y_h^k - \E[y_h^k|\cG_{h-1}^k]\right)^2$. We use the notation $\E[m_h^k|\cG_{h-1}^k] = \Var[y_h^k|\cG_{h-1}^k]= \sigma^2$. Moreover, note that $(m_h^k-\E[m_h^k|\cG_{h-1}^k])_{k,h}$ is a martingale difference sequence adapted to the filtration $\cG_h^k$, with absolute values $|m_h^k-\E[m_h^k|\cG_{h-1}^k]| \le H^2$ for all $k,h$. Therefore, by the  Azuma-Hoeffding inequality, with probability at least $1-\delta/2^{i}$,
\begin{align*}
    \sum_{k=1}^{\tau_{i-1}}\sum_{h=1}^H m_h^k \le  \sum_{k=1}^{\tau_{i-1}}\sum_{h=1}^H \E[m_h^k|\cG_{h-1}^k] + H^2\sqrt{2\tau_{i-1}H\log(2^{i}/\delta)}.
\end{align*}
Now, using a union bound, along with the definition of $T^{(i)}_m$, with probability at least $1-\delta$, for any class index $m \ge m^*$, we have
\begin{align}
\forall i \ge 1,\quad  T_m^{(i)} \le \sigma^2  + H^{3/2}\sqrt{\frac{2\log(2^{i}/\delta)}{\tau_{i-1}}} 
  \label{eq:realizable}  
\end{align}

Let us now look at the definition of $T^{(i)}_m$. We have
\begin{align*}
     T_m^{(i)} = \frac{1}{\tau_{i-1}H}\sum_{k=1}^{\tau_{i-1}}\sum_{h=1}^{H} \left(V_{h+1}^k(s_{h+1}^k)-( \widehat{P}^{(i)}_m V_{h+1}^k)(s_h^k,a_h^k) \right)^2.
\end{align*}
We write
\begin{align*}
    & \frac{1}{\tau_{i-1}H}\sum_{k=1}^{\tau_{i-1}}\sum_{h=1}^{H} \left(V_{h+1}^k(s_{h+1}^k)-( \widehat{P}^{(i)}_m V_{h+1}^k)(s_h^k,a_h^k) \right)^2 - \frac{1}{\tau_{i-1}H}\sum_{k=1}^{\tau_{i-1}}\sum_{h=1}^{H} \left(V_{h+1}^k(s_{h+1}^k)-( \widehat{P}^* V_{h+1}^k)(s_h^k,a_h^k) \right)^2 \\
    & = \frac{1}{\tau_{i-1}H}\sum_{k=1}^{\tau_{i-1}}\sum_{h=1}^{H} \left[ (\widehat{P}^* V_{h+1}^k(s_h^k,a_h^k) - \widehat{P}^{(i)}_m V_{h+1}^k)(s_h^k,a_h^k))(2V_{h+1}^k(s_{h+1}^k) - (\widehat{P}^{(i)}_m V_{h+1}^k)(s_h^k,a_h^k) - (\widehat{P}^* V_{h+1}^k)(s_h^k,a_h^k) \right]
\end{align*}
Taking expectations (w.r.t the true model $P^*$), we obtain
\begin{align*}
    \mathbb{E}T^{(i)}_m \geq \sigma^2,
\end{align*}
and hence, using similar martingale difference construction, we obtain
\begin{align*}
    T^{(i)}_m \geq \sigma^2 - H^{3/2}\sqrt{\frac{2\log(2^{i}/\delta)}{\tau_{i-1}}}
\end{align*}
with probability at least $1-\delta$.

\paragraph{2. Non-realizable model classes:} Fix a class index $m < m^*$. In this case, the true model $P^* \notin \cP_m$. We can decompose the empirical risk at epoch $i$ as $T_m^{(i)}=T^{(i)}_{m,1}+T^{(i)}_{m,2}+T^{(i)}_{m,3}$, where
\begin{align*}
    T^{(i)}_{m,1}  &= \frac{1}{\tau_{i-1}H}\sum_{k=1}^{\tau_{i-1}}\sum_{h=1}^{H}  \left(V_{h+1}^k(s_{h+1}^k)-( P^* V_{h+1}^k)(s_h^k,a_h^k) \right)^2,\\
    T^{(i)}_{m,2} & = \frac{1}{\tau_{i-1}H}\sum_{k=1}^{\tau_{i-1}}\sum_{h=1}^{H} \left(( P^* V_{h+1}^k)(s_h^k,a_h^k)-( \hat P^{(i)}_m V_{h+1}^k)(s_h^k,a_h^k) \right)^2,\\
    T^{(i)}_{m,3} &= \frac{1}{\tau_{i-1}H}\sum_{k=1}^{\tau_{i-1}}\sum_{h=1}^{H} 2 \left(( P^* V_{h+1}^k)(s_h^k,a_h^k)-( \hat P^{(i)}_m V_{h+1}^k)(s_h^k,a_h^k) \right)\left( V_{h+1}^k(s_{h+1}^k) - (P^*V_{h+1}^k)(s_h^k,a_h^k)\right).
\end{align*}
First, using a similar argument as in \eqref{eq:realizable}, with probability at least $1-\delta$, we obtain
\begin{align*}
    \forall i \ge 1, \quad T^{(i)}_{m,1} \ge \sigma^2 - H^{3/2}\sqrt{\frac{2\log(2^{i}/\delta)}{\tau_{i-1}}}.
\end{align*}
Next, by Assumption \ref{ass:sep}, we have \begin{align*}
\forall i \ge 1, \quad    T^{(i)}_{m,2} \ge \Delta~.
\end{align*}
Now, we turn to bound the term $T^{(i)}_{m,3}$. We consider both the cases -- when $\cP$ is finite and when $\cP$ is infinite.

\paragraph{Case 1 -- finite model classes:}
Note that $\hat P_m^{(i)} \in \cP_m$. Then, from \eqref{eq:finite}, we have
\begin{align*}
     \forall i \ge 1,\quad T_{m,3}^{(i)} \ge 
    -\frac{4H}{\tau_{i-1}} \log\left(\frac{|\cP_m|}{\delta}\right) -\frac{1}{2} T_{m,2}^{(i)}
\end{align*}
with probability at least $1-\delta$. Now, combining all the three terms together and using a union bound, we obtain the following for any class index $m \leq m^*-1$:
\begin{align}
   \forall i \ge 1,\quad T_m^{(i)} \ge \sigma^2 + \frac{1}{2}\Delta - H^{3/2}\sqrt{\frac{2\log(2^{i}/\delta)}{\tau_{i-1}}}  -\frac{4H}{\tau_{i-1}} \log\left(\frac{|\cP_m|}{\delta}\right).
   \label{eq:finite-non-realizable}
 \end{align}
with probability at least $1-2\delta$.

\paragraph{Case 2 -- infinite model classes:}
We follow a similar procedure, using \eqref{eq:infinite}, to obtain the following for any class index $m \leq m^*-1$:
\begin{align}
     \forall i \ge 1,\quad T_{m}^{(i)} &\ge \sigma^2 + \frac{1}{2}\Delta - H^{3/2}\sqrt{\frac{2\log(2^{i}/\delta)}{\tau_{i-1}}}
    -\frac{4H}{\tau_{i-1}} \log\left(\frac{\cN(\cP_m,\alpha,\norms{\cdot}_{\infty,1})}{\delta}\right)\nonumber\\ &\quad\quad- \alpha \left( 2 H^2 \sqrt{2 \log \left(\frac{2\tau_{i-1}H(\tau_{i-1}H+1)}{\delta} \right)} + 4H^2\right)
    \label{eq:infinite-non-realizable}
\end{align}
with probability at least $1-3\delta$. 

\subsubsection{Proof of Lemma \ref{lem:gen_infinite}}
We are now ready to prove Lemma \ref{lem:gen_infinite}, which presents the model selection Guarantee of \texttt{ARL-GEN} for infinite model classes $\lbrace \cP_m \rbrace_{m \in [M]}$. Here, at the same time, we prove a similar (and simpler) result for finite model classes also. 

First, note that we consider doubling epochs $k_i=2^i$, which implies $\tau_{i-1} = \sum_{j=1}^{i-1}k_j = \sum_{j=1}^{i-1} 2^j = 2^i -1$. With this, the number of epochs is given by $N=\lceil \log_2(K+1) -1 \rceil = \mathcal{O}(\log K)$. Let us now consider finite model classes $\lbrace \cP_m \rbrace_{m \in [M]}$.

\paragraph{Case 1 -- finite model classes:}

First, we combine \eqref{eq:realizable}
 and \eqref{eq:finite-non-realizable}, and take a union bound over all $m \in [M]$ to obtain
\begin{align*}
    &\forall m \ge m^*,\;\forall i \ge 1, \, \sigma^2  - H^{3/2}\sqrt{\frac{2N\log(2/\delta)}{\tau_{i-1}}} \leq T_m^{(i)} \le \sigma^2  + H^{3/2}\sqrt{\frac{2N\log(2/\delta)}{\tau_{i-1}}}\;\; \text{and}\\
    &\forall m \le m^*-1,\;\forall i \ge 1,\quad T_m^{(i)} \ge \sigma^2 + \frac{1}{2}\Delta - H^{3/2}\sqrt{\frac{2N\log(2/\delta)}{\tau_{i-1}}}  -\frac{4H}{\tau_{i-1}} \log\left(\frac{|\cP_M|}{\delta}\right) 
\end{align*}
with probability at least $1-2M\delta$, where we have used that $\log(2^i/\delta) \le N\log(2/\delta)$ for all $i$ and $|\cP_m| \le |\cP_M|$ for all $m$.
Now, suppose for some epoch $i^*$,  satisfies 
\begin{align*}
    2^{i^*} \geq C\max \left \lbrace \frac{2H^3 \log K}{\Delta^2} \log(2/\delta), 4H \log\left(\frac{|\mathcal{P}_M|}{\delta}\right) \right \rbrace.
\end{align*}
where $C$ is a sufficiently large universal constant. Then, we have
\begin{align*}
    &\forall m \ge m^*,\;\forall i \ge i^*,\quad \sigma^2 - \frac{c_0}{2^{i/2}} \leq T_m^{(i)} \le \sigma^2 + \frac{c_0}{2^{i/2}} \;\; \text{and}\\
    &\forall m \le m^*-1,\;\forall i \ge i^*,\quad T_m^{(i)} \ge \sigma^2 + \frac{1}{2} \Delta  - \frac{c_1}{2^{i/2}}
\end{align*}
with probability at least $1-3M\delta$.
Note that with the chosen threshold $\gamma_i$, for all $i \geq i^*$
\begin{align*}
  \sigma^2 - \frac{c_0}{2^{i/2}} + \frac{\sqrt{i}}{2^{i/2}} \leq \gamma_i &\leq \sigma^2 +\frac{c_0}{2^{i/2}} + \frac{\sqrt{i}}{2^{1/2}} \\\\
    &\leq\sigma^2 + \frac{1}{2}\Delta -\frac{c_1}{2^{i/2}}.
\end{align*}
where the last inequality comes from the choice of $2^{i^*}$. With this, we have
\begin{align*}
    \forall m \ge m^*,\;\forall i \ge i^*,\quad T_m^{(i)} &\le \gamma_i \; \text{and}\\
    \forall m \le m^*-1,\;\forall i \ge i^*,\quad T_m^{(i)} &\ge \gamma_i
\end{align*}
with probability at least $1-3M\delta$. The above equation implies that $m^{(i)}=m^*$ for all epochs $i \ge i^*$.

Now, we focus on infinite model classes $\lbrace \cP_m \rbrace_{m \in [M]}$, for which Lemma 1 is stated.

\paragraph{Case 2 -- infinite model classes:} 

First, we combine \eqref{eq:realizable}
 and \eqref{eq:infinite-non-realizable}, and take a union bound over all $m \in [M]$ to obtain
\begin{align*}
    & \forall m \ge m^*,\;\forall i \ge 1,\sigma^2  - H^{3/2}\sqrt{\frac{2N\log(2/\delta)}{\tau_{i-1}}} \leq  T_m^{(i)} \le \sigma^2  + H^{3/2}\sqrt{\frac{2N\log(2/\delta)}{\tau_{i-1}}}\;\; \text{and}\\
    & \forall m \le m^*-1,\;\forall i \ge 1,\quad T_m^{(i)} \ge \sigma^2 + \frac{1}{2}\Delta - H^{3/2}\sqrt{\frac{2N\log(2/\delta)}{\tau_{i-1}}}  -\frac{4H}{\tau_{i-1}} \log\left(\frac{\cN(\cP_M,\alpha,\norms{\cdot}_{\infty,1})}{\delta}\right)\nonumber\\&\quad\quad - \alpha \left( 2 H^2 \sqrt{2 \log \left(\frac{2K H(KH+1)}{\delta} \right)} + 4H^2\right)
\end{align*}
with probability at least $1-3M\delta$. Suppose for some epoch $i^*$, we have,
\begin{align*}
    2^{i^*} \geq C \max \left \lbrace \frac{H^3 \log K}{\Delta^2} \log(2/\delta), 4H \log\left(\frac{\cN(\cP_M,\alpha,\norms{\cdot}_{\infty,1})}{\delta}\right) \right \rbrace
\end{align*}
where $C>1$ is a sufficiently large universal constant. Then, with the choice of threshold $\gamma_i$, and doing the same calculation as above, we obtain
\begin{align*}
    \forall m \ge m^*,\;\forall i \ge i^*,\quad T_m^{(i)} &\le \gamma_i \; \text{and}\\
    \forall m \le m^*-1,\;\forall i \ge i^*,\quad T_m^{(i)} &\ge \gamma_i
\end{align*}
with probability at least $1-3M\delta$. The above equation implies that $m^{(i)}=m^*$ for all epochs $i \ge i^*$, proving the result.

\subsection{Regret Bound of \texttt{ARL-GEN} (Proof of Theorem \ref{thm:general})}
Lemma~\ref{lem:gen_infinite} implies that as soon as we reach epoch $i^*$, \texttt{ARL-GEN} identifies the model class with high probability, i.e., for each $i \ge i^*$, we have $m^{(i)}=m^*$. However, before that, we do not have any guarantee on the regret performance of \texttt{ARL-GEN}. Since at every episode the regret can be at most $H$, 
the cumulative regret up until the $i^*$ epoch is upper bounded by $\tau_{i^*-1}H$, 
which is at most $\mathcal{O} \left( \max \left \lbrace H^4 \log(K) \log(1/\delta), H^2 \log\left(\frac{\cN(\cP_M,\frac{1}{KH},\norms{\cdot}_{\infty,1})}{\delta}\right) \right \rbrace \right)$, if the model classes are infinite, and $\mathcal{O} \left( \max \left \lbrace H^4 \log(K) \log(1/\delta), H^2 \log\left(\frac{|\cP_M|}{\delta}\right) \right \rbrace \right)$, if the model classes are finite. Note that this is the cost we pay for model selection. 

Now, let us bound the regret of \texttt{ARL-GEN} from epoch $i^*$ onward. 
Let $R^{\texttt{UCRL-VTR}}(k_i,\delta_i,\cP_{m^{(i)}})$ denote the cumulative regret of \texttt{UCRL-VTR}, when it is run for $k_i$ episodes with confidence level $\delta_i$ for the family $\cP_{m^{(i)}}$. Now, using the result of \citet{ayoub2020model}, we have
\begin{align*}
    R^{\texttt{UCRL-VTR}}(k_i,\delta_i,\cP_m^{(i)}) \le 1+H^2\dimn_{\cE}\left(\cF_{m^{(i)}},\frac{1}{k_iH}\right) &+4\sqrt{\beta_{k_i}(\cP_{m^{(i)}},\delta_i)\dimn_{\cE}\left(\cF_{m^{(i)}},\frac{1}{k_iH}\right)k_i H}\\ &+H\sqrt{2k_iH\log(1/\delta_i)}
\end{align*}
with probability at least $1-2\delta_i$.
With this and Lemma \ref{lem:gen_infinite}, the regret of \texttt{ARL-GEN} after $K$ episodes (i.e., after $T=KH$ timesteps) is given by
\begin{align*}
    R(T) &\leq \tau_{i^*-1}H + \sum_{i=i^*}^N R^{\texttt{UCRL-VTR}}(k_i,\delta_i,\cP_{m^{(i)}}) \\
     & \leq \tau_{i^*-1}H + N+ \sum_{i=i^*}^N H^2\dimn_{\cE}\left(\cF_{m^*},\frac{1}{k_iH}\right) +4 \sum_{i=i^*}^N \sqrt{\beta_{k_i}(\cP_{m^*},\delta_i)\dimn_{\cE}\left(\cF_{m^*},\frac{1}{k_iH}\right)k_i H}\\
     &\quad\quad\quad\quad\quad+ \sum_{i=i^*}^N H\sqrt{2k_iH\log(1/\delta_i)}~.
\end{align*}
The above expression holds with probability at least $1-3M\delta-2\sum_{i=i^*}^{N}\delta_i$ for infinite model classes, and with probability at least $1-2M\delta-2\sum_{i=i^*}^{N}\delta_i$ for finite model classes. Let us now compute the last term in the above expression. Substituting $\delta_i = \delta/2^i$, we have
\begin{align*}
    \sum_{i=i^*}^N H\sqrt{2k_iH\log(1/\delta_i)} &= \sum_{i=i^*}^N H\sqrt{2k_i H \,i \,\log(2/\delta)} \\
     & \leq H\sqrt{2HN\log(2/\delta)} \sum_{i=1}^N \sqrt{k_i} \\
     &= \mathcal{O}\left(  H \sqrt{KH N \log(1/\delta)}\right) =  \mathcal{O}\left( H \sqrt{T\log K \log(1/\delta)}\right),
\end{align*}
where we have used that the total number of epochs $N=\cO(\log K)$, and that
\begin{align*}
    \sum_{i=1}^N \sqrt{k_i} & = \sqrt{k_N}\left(1 + \frac{1}{\sqrt{2}} + \frac{1}{2} + \ldots N\text{-th term} \right)\\
    & \leq \sqrt{k_N}\left(1 + \frac{1}{\sqrt{2}} + \frac{1}{2} + ... \right)= \frac{\sqrt{2}}{\sqrt{2} -1} \sqrt{k_N}  \leq \frac{\sqrt{2}}{\sqrt{2} -1} \sqrt{K}.
    \end{align*}
     Next, we can upper bound the third to last term in the regret expression as
  \begin{align*}
      \sum_{i=i^*}^N H^2\dimn_{\cE}\left(\cF_{m^*},\frac{1}{k_iH}\right) \le H^2 N \dimn_{\cE}\left(\cF_{m^*},\frac{1}{KH}\right) = \cO \left(H^2d_{\cE}^* \log K \right).
  \end{align*}
    Now, notice that, by substituting $\delta_i = \delta/2^i$, we can upper bound $\beta_{k_i}(\cP_{m^*},\delta_i)$ as follows:
  \begin{align*}
      \beta_{k_i}(\cP_{m^*},\delta_i) &=  \mathcal{O}\left( H^2 \, i \, \log \left( \frac{\cN(\mathcal{P}_{m^*},\frac{1}{k_i H},\norms{\cdot}_{\infty,1})}{\delta}\right) + H^2 \left(1+ \sqrt{i\log \frac{ k_iH(
    k_i H+1)}{\delta}} \right) \right)\\ & \leq \mathcal{O}\left( H^2 \, N \, \log \left( \frac{\cN(\mathcal{P}_{m^*},\frac{1}{K H},\norms{\cdot}_{\infty,1})}{\delta}\right) + H^2 \left(1+ \sqrt{N \log \frac{K H}{\delta}} \right) \right) \\
      & \leq \mathcal{O}\left( H^2 \, \log K \, \log \left( \frac{\cN(\mathcal{P}_{m^*},\frac{1}{K H},\norms{\cdot}_{\infty,1})}{\delta}\right) + H^2  \sqrt{ \log K \log (K H/\delta)}  \right)
  \end{align*}
  for infinite model classes, and $\beta_{k_i}(\cP_{m^*},\delta_i) = \cO\left(H^2\log K\log\left(\frac{|\cP_{m^*}|}{\delta} \right) \right)$ for finite model classes.
  With this, the second to last term in the regret expression can be upper bounded as 
  \begin{align*}
    & \sum_{i=i^*}^N 4\sqrt{\beta_{k_i}(\cP_{m^*},\delta_i)\dimn_{\cE}\left(\cF_{d_i},\frac{1}{k_iH}\right)k_i H}  \\
    & \leq \mathcal{O} \left(H\sqrt{  \log K \, \log \left( \frac{\cN(\mathcal{P}_{m^*},\frac{1}{K H},\norms{\cdot}_{\infty,1})}{\delta}\right) + \sqrt{ \log K \log (K H/\delta)}  } \,\,\sqrt{\dimn_{\cE}\left(\cF_{m^*},\frac{1}{K H}\right)} \sum_{i=i^*}^N \sqrt{k_i H} \right)\\
    & \leq \mathcal{O} \left(H\sqrt{ \log K \, \log \left( \frac{\cN(\mathcal{P}_{m^*},\frac{1}{K H},\norms{\cdot}_{\infty,1})}{\delta}\right) \log\left(\frac{KH}{\delta}\right)} \,\,\sqrt{KH\,\,\dimn_{\cE}\left(\cF_{m^*},\frac{1}{K H}\right)}\right)\\
    & = \cO \left(H\sqrt{Td_{\cE}^*(\mathbb{M}^*+\log(1/\delta))\log K \log(T/\delta)} \right)
  \end{align*}
  for infinite model classes. Similarly, for finite model classes, we can upper bound this term by $\cO \left(H\sqrt{Td_{\cE}^*\log\left(\frac{|\cP_{m^*}|}{\delta} \right)\log K} \right)$.
 
 Hence, for infinite model classes, the final regret bound can be written as
\begin{align*}
   R(T) &= \mathcal{O} \left( \max \left \lbrace H^4 \log(K) \log(1/\delta), H^2 \log\left(\frac{\cN(\cP_M,1/T,\norms{\cdot}_{\infty,1})}{\delta}\right) \right \rbrace \right)\\ &\quad + \cO \left(H^2d_{\cE}^* \log K+H\sqrt{Td_{\cE}^*(\mathbb{M}^*+\log(1/\delta))\log K \log(T/\delta)} \right)~.
\end{align*}
 The above regret bound holds with probability greater than
 \begin{align*}
      1-3M\delta -\sum_{i=i^*}^{N}\frac{\delta}{2^{i-1}} \ge 1-3M\delta -\sum_{i\ge 1}\frac{\delta}{2^{i-1}} =1-3M\delta-2\delta~,
 \end{align*}
 which completes the proof of Theorem \ref{thm:general}.
 
Similarly, for finite model classes, the final regret bound can be written as
\begin{align*}
    R(T) &= \mathcal{O} \left( \max \left \lbrace H^4 \log(K) \log(1/\delta), H^2 \log\left(\frac{|\cP_M|}{\delta}\right) \right \rbrace \right)\\ & \quad+  \cO \left(H^2d_{\cE}^* \log K+H\sqrt{Td_{\cE}^*\log\left(\frac{|\cP_{m^*}|}{\delta} \right)\log K} \right),
\end{align*}
which holds with probability greater than $1-2M\delta-2\delta$.

\section{DETAILS FOR SECTION \ref{sec:lin}}

In this setting, we do not need any separability condition like Assumption~\ref{ass:sep}. Instead, we make the following assumptions on the feature map $\phi$. To this end, for any function $V:\mathcal{S} \to \R$ and state-action pair $(s,a)$, we first define the function $\phi_V(s,a) := \int_{\cS}\phi(s,a,s')V(s')ds'$. We then introduce a function class $\cV$, which is the collection of all possible maps
\begin{align*}
     s \!\mapsto\! \min \!\left\lbrace\! H, \max_{a\in \cA} \!\Big( r(s,a) \!+\! \inprod{\psi(s,a)}{\!\theta} \!+\! \eta\! \norms{\psi(s,a)}_W \!\Big) \!\right\rbrace\!,
\end{align*}
where $\psi:\cS \times \cA \to \mathbb{B}_d^1$, $\theta \in \R^d$, $ W \in \mathbb{S}_{d}^+$ and $\eta > 0$ parameterize these maps. Note that the value functions computed by our algorithms will belong to this class.
\begin{assumption}
\label{asm:eigen}
For any bounded function $V:\mathcal{S} \to [0,H]$ and any state-action pair $(s,a) \in \mathcal{S}\times \mathcal{A}$, we have $\eucnorm{\phi_V(s,a)}\le 1$. Furthermore, there exists a $\Sigma \in \mathbb{S}_d^{+}$ and a
$\rho_{\min} \!>\! 0$ such that for all $V \!\in\! \cV$, $h \!\in\! [H]$, $k \!\in\! [K]$, the following holds almost surely:
\begin{align*}
\E \left[ \phi_{V}(s^k_h,a^k_h) \phi_{V} (s^k_h,a^k_h)^\top |\cG'_{k-1} \right] = \Sigma \succeq \rho_{\min}\, I\;,
\end{align*}
where $(s_h^k,a_h^k)$ is the state-action pair visited in step $h$ of episode $k$ and $\cG'_{k-1}$ denotes the $\sigma$-field summarizing the information available before the start of episode $k$.
\end{assumption}
We emphasize that similar assumptions have featured in model selection for stochastic contextual bandits \citep{osom,foster_model_selection,ghosh2021problem}. The assumption ensures that the confidence ball of the estimate of $\theta^*$ shrinks at a certain rate. We emphasize here that for model selection problems the assumption becomes crucial since we care about the rate of estimate shrinkage as well as the regret performance. Note that classical algorithms, like OFUL for bandits \citep{abbasi2011improved}, UCRL-VTR-LIN for RL \citep{ayoub2020model} only care about regret minimization and hence conditions similar to the above assumption are not needed. 
The above observation and similar eigenvalue assumption first featured in \cite{osom}, for model selection between linear and standard bandits. Later \citet{foster_model_selection,ghosh2021problem} used this for model selection in (linear) contextual bandits. Recently, \cite{ghosh2021collaborative} use this condition for careful model estimation and clustering for contextual bandit problems. In fact the necessity of an assumption like this is posed as an open problem in \citet{foster2020open}.

First, using Assumption~\ref{asm:eigen}, we obtain the following result, which is crucial to understand the rate at which the confidence ellipsoid of $\theta^*$ shrinks.
\begin{lemma}
\label{lem:eigenvalue}
Fix a $\delta \in (0,1]$, and suppose that Assumption~\ref{asm:eigen} holds.
Then, with probability at least $1-\delta$, uniformly over all $k \in \lbrace \tau_{\min}(\delta),\ldots,K \rbrace$, we have 
$\gamma_{\min} (\Sigma_{k})  \geq 1 +  \frac{\rho_{\min}kH}{2}$, where  $\tau_{\min}(\delta) := \left( \frac{16}{\rho_{\min}^2} + \frac{8}{3 \rho_{\min}} \right) \log\left(\frac{2dKH}{\delta}\right)$.
\end{lemma}
\begin{proof}
We follow a similar proof technique as used in \cite{osom} in the setting of contextual linear bandits using Assumption \ref{asm:eigen} and the matrix Freedman inequality. First, note that by Assumption \ref{asm:eigen}, we have $\norms{\phi_{V_{h+1}^k}(s_h^k,a_h^k)} \le 1$, and $\E[\phi_{V_{h+1}^k}(s_h^k,a_h^k)\phi_{V_{h+1}^k}(s_h^k,a_h^k)^\top | \cG'_{k-1}]=\Sigma \succeq \rho_{min} I$ for all $h \in [H]$ and $k \in [K]$. Now, fix an $h \in [H]$, and define the following matrix martingale
\begin{align*}
    Z_{h,k} = \sum_{j=1}^k  \phi_{V_{h+1}^j}(s_h^j,a_h^j)\phi_{V_{h+1}^j}(s_h^j,a_h^j)^\top - k\Sigma~,
\end{align*}
with $Z_{h,0} = 0$. Next, consider the martingale difference sequence
\begin{align*}
    Y_{h,k} = Z_{h,k} -Z_{h,k-1}=\phi_{V_{h+1}^k}(s_h^k,a_h^k)\phi_{V_{h+1}^k}(s_h^k,a_h^k)^\top - \Sigma~.
\end{align*}
Since $\|\phi_{V_{h+1}^k}(s_h^k,a_h^k)\| \leq 1$, we have 
\begin{align*}
    \|\Sigma\|_{op} = \| \E  [\phi_{V_{h+1}^k}(s_h^k,a_h^k)\phi_{V_{h+1}^k}(s_h^k,a_h^k)^\top |\cG'_{k-1}] \|_{op} \leq  1,
\end{align*}
and as a result
\begin{align*}
    \|Y_{h,k}\|_{op}=\left\| \phi_{V_{h+1}^k}(s_h^k,a_h^k)\phi_{V_{h+1}^k}(s_h^k,a_h^k)^\top - \Sigma\right\|_{op} \leq  2.
\end{align*}
Furthermore, a simple calculation yields
\begin{align*}
    \|\E[Y_{h,k} Y_{h,k}^\top | \cG'_{k-1}]\|_{op} &=  \|\E[Y_{h,k}^\top Y_{h,k}|\cG'_{k-1}]\|_{op}\\ &= \left\| \E\left[\|\phi_{V_{h+1}^k}(s_h^k,a_h^k)\|^2\phi_{V_{h+1}^k}(s_h^k,a_h^k)\phi_{V_{h+1}^k}(s_h^k,a_h^k)^\top | \cG'_{k-1}\right] -\Sigma^2\right\| \leq 2.
\end{align*}
Now, applying matrix Freedman inequality (Theorem 13 of \cite{osom}) with $R=2, \omega^2 = 2k, u =\rho_{\min}k/2$, we obtain
\begin{align*}
    \Prob\left[ \|Z_{h,k}\|_{op} \geq \frac{\rho_{\min}k}{2} \right] \leq \frac{\delta}{KH},
\end{align*}
for any $k\geq \left( \frac{16}{\rho_{\min}^2} + \frac{8}{3 \rho_{\min}} \right) \log(2dKH/\delta)$. Note that the above concentration bound holds for any $h \in [H]$. Now, we define $Z_k=\sum_{h=1}^{H}Z_{h,k}$.
Then, applying a union bound, we obtain
\begin{align*}
    \norms{Z_k}_{op} \le \sum_{h=1}^{H} \norms{Z_{h,k}}_{op} \le \frac{\rho_{\min}kH}{2}
\end{align*}
for a given $k \in \lbrace \tau_{\min}(\delta),\ldots,K\rbrace$, with probability at least $1-\delta/K$. By Assumption~\ref{asm:eigen}, we have $\gamma_{\min}(kH\Sigma) \geq \rho_{\min}kH$. Hence, using Weyl's inequality, we obtain 
\begin{align*}
    \gamma_{\min}\left( \sum_{j=1}^k\sum_{h=1}^{H}  \phi_{V_{h+1}^j}(s_h^j,a_h^j)\phi_{V_{h+1}^j}(s_h^j,a_h^j)^\top \right) \ge \rho_{\min}Hk/2
\end{align*}
for a given $k \in \lbrace \tau_{\min}(\delta),\ldots,K\rbrace$, with probability at least $1-\delta/K$. 
Now, the result follows by taking a union bound, and by noting that $\Sigma_k=I+ \sum_{j=1}^k\sum_{h=1}^{H}  \phi_{V_{h+1}^j}(s_h^j,a_h^j)\phi_{V_{h+1}^j}(s_h^j,a_h^j)^\top$.
\end{proof}


\subsection{Analysis of \texttt{ARL-LIN(dim)}}
\label{app:dim}

First, we prove the following concentration result on the estimates of $\theta^*$ in the sup-norm, which is important in designing the model selection procedure of \texttt{ARL-LIN(dim)}.

\begin{lemma}
Suppose, Assumption \ref{asm:eigen} holds. Also, suppose that $\widehat{\theta}_{\tau}$ is the estimate of $\theta^*$ after running {{\ttfamily UCRL-VTR-LIN}} in full $d$-dimension for $\tau$ episodes with norm upper bound $b$ and confidence level $\delta$, where $\tau \in \lbrace\tau_{\min}(\delta),\ldots,K\rbrace$ and $\tau_{\min}(\delta) = \left( \frac{16}{\rho_{\min}^2} + \frac{8}{3 \rho_{\min}} \right) \log(2dKH/\delta)$. Furthermore, for any $\varepsilon \in (0,1)$, let $\tau = \Omega \left( \frac{b^2Hd \log^2(K^2 H/\delta)\log(KH)}{\rho_{\min} \varepsilon^2}\right)$. Then, we have
\begin{align*}
    \mathbb{P}\left[\norms{\widehat{\theta}_\tau - \theta^*}_{\infty} \geq \varepsilon\right] \leq 2 \delta.
\end{align*}
\label{lem:est_bound}
\end{lemma}

\begin{proof}
By Lemma~\ref{lem:eigenvalue} and the properties of {{\ttfamily UCRL-VTR-LIN}}, we obtain for all $\tau \in \lbrace \tau_{\min}(\delta),\ldots,K\rbrace$,
\begin{align*}
  \norms{\widehat{\theta}_\tau - \theta^*}_{\infty} \le \norms{\widehat{\theta}_\tau - \theta^*} \le \frac{\sqrt{\beta_\tau(\delta)}}{\sqrt{1+\rho_{\min}H\tau/2}}
\end{align*}
with probability at least $1-2\delta$. Let us now look at the confidence radius
\begin{align*}
    \beta_{\tau}(\delta) = O\left(b^2 +  H^2d\log(\tau H)\log^2(\tau^2H/\delta)\right) \le C\, b^2 H^2d\log(K H)\log^2(K^2H/\delta).
\end{align*}
for some positive constant $C$. With this, the above equation can be written as 
\begin{align*}
  \norms{\widehat{\theta}_\tau - \theta^*}_{\infty} 
  \le \sqrt{C}\,\frac{bH\sqrt{d\log(K H)\log^2(K^2H/\delta)}}{\sqrt{\rho_{\min}H\tau/2}}~.
\end{align*}
Now setting $\tau \ge  2C\,\frac{b^2Hd \log^2(K^2 H/\delta)\log(KH)}{\rho_{\min} \varepsilon^2}$, and using the fact that $\tau < K$, we get the result.
\end{proof}

Now, we have the following concentration result on the estimates $\widehat\theta^{(i)}$, which essentially implies that the support estimation phase succeeds with high probability after a certain number of epochs.
\begin{lemma}\label{lem:support-est}
Suppose Assumption~\ref{asm:eigen} holds, and for any $\delta \in (0,1]$, the initial phase length $k_0$ satisfies
$\sqrt{k_0} = \tau_{\min}(\delta)+ O \left( \frac{b^2Hd \log^2(K^2 H/\delta)\log(KH)}{\rho_{\min} (0.5)^2}\right)$, where $\tau_{\min}(\delta)$ is as defined in Lemma~\ref{lem:eigenvalue}.
Then, for all epochs $i \geq 10$, we have $\mathbb{P} \left[ || \widehat{\theta}^{(i)} - \theta^*||_{\infty} \geq (0.5)^i \right] \leq \frac{\delta}{2^{i-1}}$.   
\end{lemma}

\begin{proof}
Note that, for each epoch $i \ge 1$, $\widehat{\theta}^{(i)}$ is computed by considering the samples of {{\ttfamily UCRL-VTR-LIN}} over $\tau_{i-1}$ episodes, where $\tau_{i-1}=\sum_{j=0}^{i-1}6^j \lceil \sqrt{k_0} \rceil$.
If $\tau_{\min}(\delta) \leq \tau_{i-1} \leq K$ and $\tau_{i-1}= \Omega \left( \frac{b^2Hd \log^2(K^2 H2^{i}/\delta)\log(KH)}{\rho_{\min} (0.5)^{2i}}\right)$, then by Lemma \ref{lem:est_bound}, we have
\begin{align*}
    \mathbb{P} \left[ || \widehat{\theta}^{(i)} - \theta^*||_{\infty} \geq (0.5)^i \right] \leq \frac{\delta}{2^{i-1}}.
\end{align*}
Consider the epoch $i=1$. In this case, $\widehat \theta^{(1)}$ is computed by samples of {{\ttfamily UCRL-VTR-LIN}} over $\tau_0=\lceil \sqrt{k_0} \rceil$ episodes. Then, the choice $\sqrt{k_0} = \tau_{\min}(\delta)+ \cO \left( \frac{b^2Hd \log^2(K^2 H/\delta)\log(KH)}{\rho_{\min} (0.5)^2}\right)$ (we have added $\tau_{\min}(\delta)$ to make the calculations easier), ensures that 
\begin{align*}
    \mathbb{P}\left[||\widehat{\theta}^{1} - \theta^*||_{\infty} \geq 0.5 \right] \leq \delta~.
\end{align*}
Note that, we require $\tau_{i-1} \geq i\, 4^i \lceil\sqrt{k_0}\rceil$.
The proof is concluded if we can show that this holds for epochs $i \geq 10$. To this end, we note that
\begin{align*}
    \tau_{i-1}=\sum_{j=0}^{i-1}6^j \lceil \sqrt{k_0} \rceil  = \lceil \sqrt{k_0} \rceil \frac{(6^i-1)}{5}
    \geq i\, 4^i \lceil \sqrt{k_0} \rceil,
\end{align*}
where the last inequality holds for all $i \geq 10$.
\end{proof}

Armed with the above results, we have the following regret bound for \texttt{ARL-LIN(dim)}.
\begin{theorem}[Cumulative regret of \texttt{ARL-LIN(dim)}]\label{thm:dim}
Suppose \texttt{ARL-LIN(dim)} is run with parameter $k_0$ chosen as in Lemma~\ref{lem:support-est} for $K$ episodes. Then, with probability at least $1-3\delta$, its regret 
\begin{align*}
    R(T) \!=\! \Tilde{\mathcal{O}}\!\left(\! \frac{Hk_0}{\gamma^{5.18}}    \!+\! \!\left(\!bd^*\!\sqrt{\!H^3T} \!+\! b\,dH^2K^{1/4}\!\right)\!\!\polylog (T\!/\delta)   \!\!\right)\!\!,
\end{align*}
where $\gamma =  \min\{|\theta^*(j)|: \theta^*(j) \neq 0\} $ with $\theta^*(j)$ denoting the $j$-th coordinate of $\theta^*$.
\end{theorem}

\begin{proof}
We first calculate the probability of the event $\mathcal{E}=\left\lbrace\bigcap_{i \geq 10} \left\{||  \widehat{\theta}^{(i)} - \theta^*||_{\infty} \leq (0.5)^i \right\}\right\rbrace$, which follows from Lemma 4 by a straightforward union bound. Specifically, we have
\begin{align*}
        \mathbb{P}[\cE]=\mathbb{P} \left[ \bigcap_{i \geq 10} \left\{||  \widehat{\theta}^{(i)} - \theta^*||_{\infty} \leq (0.5)^i \right\} \right] &= 1 - \mathbb{P} \left[ \bigcup_{i \geq 10}\left\{|| \widehat{\theta}^{(i)} - \theta^*||_{\infty} \geq (0.5)^i \right\} \right]\\
        &\geq 1 - \sum_{i \geq 10} \mathbb{P} \left[ ||\widehat{\theta}^{(i)}-\theta^*||_{\infty} \geq (0.5)^i \right]\\
        &\geq 1 - \sum_{i\ge 10} \frac{\delta}{2^{i-1}}
        \geq 1 - \sum_{i\geq 10}\frac{\delta}{2^{i-1}}
        =1-\delta.
\end{align*}
Now, consider the phase $i(\gamma):= \max \left\lbrace 10,\log_{2} \left( \frac{1}{\gamma} \right)\right\rbrace$. Note that when event $\mathcal{E}$ holds, then for all epochs $i \geq i(\gamma)$, $\mathcal{D}^{(i)}$ is the correct set of $d^*$ non-zero coordinates of $\theta^*$. Thus, with probability at least $1-\delta$, the cumulative regret of \texttt{ARL-LIN(dim)} after $K$ episodes is given by
\begin{align*}
    R(T)  \leq H\sum_{j=0}^{i(\gamma)-1} 36^j k_0    + \sum_{j = i(\gamma)}^{N} R^{\texttt{UCRL-VTR-LIN}}_{d^*}(36^jk_0,\delta_j,b)
    +   R_{d}^{\texttt{UCRL-VTR-LIN}}\left(\sum_{j=0}^{N}6^j \lceil\sqrt{ k_0} \rceil,\delta,b\right),
\end{align*}
where $N$ denotes the total number of epochs.
Note that $N=\mathcal{O}\left(\bigg\lceil \log_{36} \left( \frac{K}{k_0} \right) \bigg\rceil \right)$, and hence
$\sum_{j=0}^{N}6^j \lceil\sqrt{ k_0}\rceil = O \left(\sqrt{K}\right)$.
Then, the third term in the above regret expression can be upper bounded by $R_{d}^{\texttt{UCRL-VTR-LIN}}(\sqrt{K},\delta,b)$. Here, the subscript $d$ denotes that \texttt{UCRL-VTR-LIN} in full $d$-coordinates during the support estimation phases of all epochs $j \ge 0$. Thus, using the result of \citet{jia2020model}, the regret suffered by \texttt{ARL-LIN(dim)} during all the support estimation phases can be upper bounded as
\begin{align*}
     R_{d}^{\texttt{UCRL-VTR-LIN}}\left(\sum_{j=0}^{N}6^j \lceil\sqrt{ k_0} \rceil,\delta,b\right) &\le R_{d}^{\texttt{UCRL-VTR-LIN}}(\sqrt{K},\delta,b)\\ &= \mathcal{O}\left(b\,dH^2K^{1/4}\log(\sqrt{K}H)\log(KH/\delta) \right)
\end{align*}
with probability at least $1-\delta$.

Now, we turn to upper bound the second term of the regret expression. Here, the subscript $d^*$ denotes that \texttt{UCRL-VTR-LIN} is run in only $d^*$-coordinates (with high probability) during the regret minimization phases of all epochs $j \ge i(\gamma)$. Now, using the result of \citet{jia2020model}, for all epochs $j \ge i(\gamma)$, we have
\begin{align*}
  R_{d*}^{\texttt{UCRL-VTR-LIN}}(36^jk_0,\delta_j,b) =\mathcal{O}\left(b\,d^*H^2\sqrt{36^jk_0}\log(36^jk_0 H)\log(36^{2j}k_0^2H/\delta_j) \right)
\end{align*}
with probability at least $1-\delta_j$,
Substituting $\delta_j=\delta/2^j$, the regret suffered by \texttt{ARL-LIN(dim)} during all the regret minimization phases can be upper bounded as
\begin{align*}
    \sum_{j = i(\gamma)}^{N} R_{d^*}^{\texttt{UCRL-VTR-LIN}}(36^jk_0,\delta_j,b)   &= \sum_{j = i(\gamma)}^{N} \cO\left( bd^*H^2 \sqrt{36^j k_0} \, \poly(j) \, \log(k_0H)\log(k_0^2H/\delta) \right)\\  &\leq \cO\left( bd^*H^2 \,\, \mathsf{poly}(N)  \,\, \log(k_0H)\log(k_0^2H/\delta) \right) \sum_{j = i(\gamma)}^{N}6^j\lceil\sqrt{k_0}\rceil \\
    &= \cO\left( bd^*H^2\sqrt{K} \,\, \mathsf{polylog}(K/k_0)  \,\, \log(k_0H)\log(k_0^2H/\delta) \right)
\end{align*}
with probability at least $1 - \sum_{j \geq i(\gamma)}\delta/2^j \geq 1-\delta$.
Here, we have used that the total number of epochs $N=\mathcal{O}\left(\bigg\lceil \log_{36} \left( \frac{K}{k_0} \right) \bigg\rceil \right)=\cO(\log(K/k_0))$.

Putting everything together, the regret of \texttt{ARL-LIN(dim)} is upper bounded as
\begin{align*}
  R(T) &\leq  Hk_0 36^{i(\gamma)} +  \cO\left( bd^*\sqrt{H^3T} \,\, \mathsf{polylog}(K/k_0)  \,\, \log(k_0H)\log(k_0^2H/\delta) \right)\\ & \quad \quad+ \mathcal{O}\left(b\,dH^2K^{1/4}\log(\sqrt{K}H)\log(KH/\delta) \right)\\
  & = \cO\left( \frac{H}{\gamma^{5.18}} k_0   + \left(bd^*\sqrt{H^3T} + b\,dH^2K^{1/4}\right) \polylog(K/k_0) \polylog (T/\delta)  \right)
  \end{align*}
with probability at least $1-3\delta$. Here, in the last step, we have used that $36 \leq 2^{5.18}$, which completes the proof. 
\end{proof}



\subsection{Norm As Complexity Measure}
\label{app:norm}
In this section, we define $\norms{\theta^*}$ as a measure of complexity of the problem. In linear kernel MDPs, if $\theta^*$ is close to $0$ (with small $\norms{\theta^*}$), the set of states $s'$ for the next step will have a small cardinality. Similarly, when $\theta^*$ is away from $0$ (with large $\|\theta^*\|$), the above-mentioned cardinality will be quite large. Hence, we see that $\norms{\theta^*}$ serves as a natural measure of complexity.




\subsubsection{Algorithm: Adaptive Reinforcement Learning - Linear (norm)}
\begin{algorithm}[t!]
  \caption{Adaptive Reinforcement Learning - Linear (norm) -- \texttt{ARL-LIN(norm)}}
  \begin{algorithmic}[1]
 \STATE  \textbf{Input:} An upper bound $b$ of $\eucnorm{\theta^*}$, initial epoch length $k_1$, confidence level $\delta \in (0,1]$
 \STATE Initialize estimate of $\|\theta^*\|$ as $b^{(1)} = b$, set $\delta_1=\delta$
  \FOR{ epochs $i=1,2 \ldots $}
  \STATE Play \texttt{UCRL-VTR-LIN}  with norm estimate $b^{(i)}$ for $k_i$ episodes with confidence level $\delta_i$ 
  \STATE Refine estimate of $\|\theta^*\|$ as $ b^{(i+1)}\! =\! \max_{\theta \in \cB_{k_{i}}} \!\!\|\theta\|$
  \STATE Set $k_{i+1} = 2 k_{i}$, $\delta_{i+1} = \frac{\delta_i}{2}$
    \ENDFOR
  \end{algorithmic}
  \label{algo:norm}
\end{algorithm}
We propose and analyze an algorithm (Algorithm~\ref{algo:norm}), that adapts to the problem complexity $\|\theta^*\|$, and as a result, the regret obtained will depend on $\|\theta^*\|$. In prior work \citep{jia2020model}, usually it is assumed that $\theta^*$ lies in a norm ball with known radius, i.e., $\|\theta^*\| \leq b$. This is a non-adaptive algorithm and the algorithm uses $b$ as a proxy for the problem complexity, which can be a huge over-estimate. In sharp contrast, we start with this over estimate of $\|\theta^*\|$, and successively refine this estimate over multiple epochs. We show that this refinement strategy yields a consistent sequence of estimates of $\|\theta^*\|$, and as a consequence, our regret bound depends on $\|\theta^*\|$, but not on its upper bound $b$.

\paragraph{Our Approach:} Similar to \texttt{ARL-GEN}, we consider doubling epochs - at each epoch $i \ge 1$, \texttt{UCRL-VTR-LIN} is run for $k_i=2^{i-1}k_1$ episodes with confidence level $\delta_i=\frac{\delta}{2^{i-1}}$ and norm estimate $b^{(i)}$, where the initial epoch length $k_1$ and confidence level $\delta$ are parameters of the algorithm. We begin with $b$ as an initial (over) estimate of $\eucnorm{\theta^*}$, and at the end of the $i$-th epoch, based on the confidence set built, we choose the new estimate as $ b^{(i+1)} = \max_{\theta \in \cB_{k_{i}}} \|\theta\|$. We argue that this sequence of estimates is indeed consistent, and as a result, the regret depends on $\eucnorm{\theta^*}$.



\subsubsection{Analysis of \texttt{ARL-LIN(norm)}}



First, we present the main result of this section. We show that the norm estimates $b^{(i)}$ computed 
by \texttt{ARL-LIN(norm)} (Algorithm~\ref{algo:norm}) indeed converges to the true norm $\norms{\theta^*}$ at an exponential rate with high probability.

\begin{lemma}[Convergence of norm estimates]
\label{lem:sequence}
Suppose Assumption~\ref{asm:eigen} holds. Also, suppose that, for any $\delta \in (0,1]$, the length $k_1$ of the initial epoch satisfies 
\begin{align*}
        k_1 \!\geq\! \max\! \left\{\!\tau_{\min}(\delta) \log_2\!\left(\!1\!+\!K/k_1\!\right), C \left ( b\,\max\{p,q\}  \right)^2  d \right\},
    \end{align*}
where $p \!=\! O\!\left(\!\frac{1}{\sqrt{\rho_{\min}H}}\!\right)$, $q \!=\! O\!\left(\!\sqrt{\!\frac{H\log(k_1H)\log^2(k_1^2 H/\delta)}{\rho_{\min}}}  \right)$, $\tau_{\min}(\delta)$ is as defined in Lemma~\ref{lem:eigenvalue}, and $C > 1$ is some sufficiently large universal constant. Then, with probability exceeding $1-4\delta$, the sequence $\{b^{(i)}\}_{i=1}^\infty$ converges to $\|\theta^*\|$ at a rate $\mathcal{O}\left(\frac{i^{3/2}}{2^i}\right)$, and $b^{(i)} \leq c_1 \|\theta^*\| + c_2$, where $c_1,c_2 > 0$ are universal constants.
\end{lemma}

\begin{proof}
We consider doubling epochs, with epoch lengths  $k_i = 2^{i-1} k_1$ for $i \in \{1,\ldots,N\}$, where $k_1$ is the initial epoch length and $N$ is the number of epochs. From the doubling principle, we obtain
\begin{align*}
    \sum_{i=1}^N 2^{i-1}k_1 = K \,\, \Rightarrow \, N  = \log_2 \left(1+ \frac{K}{k_1} \right) =\cO\left(\log (K/k_1)\right).
\end{align*}
Now, let us consider the $i$-th epoch, and let $\widehat{\theta}_{k_i}$ be the least square estimate of $\theta^*$ at the end of epoch $i$, which is the estimate computed by \texttt{UCRL-VTR-LIN} after $k_i$ episodes. The confidence interval at the end of epoch $i$, i.e., after \texttt{UCRL-VTR-LIN} is run with a norm estimate $b^{(i)}$ for $k_i$ episodes with confidence level $\delta_i$, is given by
\begin{align*}
    \mathcal{B}_{k_i} = \left \lbrace \theta \in \real^d: \|\theta - \widehat{\theta}_{k_i}\|_{\Sigma_{k_i}} \leq \sqrt{\beta_{k_i}(\delta_i)} \right \rbrace.
\end{align*}
Here $\beta_{k_i}(\delta_i)$ denotes the radius and $\Sigma_{k_i}$ denotes the shape of the ellipsoid.
Using Lemma~\ref{lem:eigenvalue}, one can rewrite $\mathcal{B}_{k_i}$ as 
\begin{align*}
    \mathcal{B}_{k_i} = \left \lbrace \theta \in \real^d: \|\theta - \widehat{\theta}_{k_i}\| \leq \frac{ \sqrt{\beta_{k_i}(\delta_i)}}{\sqrt{1+\rho_{\min}Hk_i/2}} \right \rbrace,
\end{align*}
with probability at least $1-\delta_i = 1-\delta/2^{i-1}$. Here, from Lemma \ref{lem:eigenvalue}, we use the fact that $\gamma_{\min}(\Sigma_{k_i}) \geq 1+ \rho_{\min}H k_i/2$, provided $k_i \geq \tau_{\min}(\delta/2^{i-1})$. To ensure this condition, we take (the sufficient condition) $k_1 \geq \tau_{\min}(\delta) N $. 
Hence, with $k_1$ satisfying $k_1 \geq \tau_{\min}(\delta)\log_2 \left(1+ \frac{K}{k_1} \right) $, we ensure that $\gamma_{\min}(\Sigma_{k_i}) \geq 1+ \rho_{\min} Hk_i/2$. Also, we know that $\theta^* \in \mathcal{B}_{k_i}$ with probability at least $1-\delta_i$. Hence, we obtain
\begin{align*}
    \|\widehat{\theta}_{k_i}\| \leq \|\theta^*\| + \frac{ \sqrt{\beta_{k_i}(\delta_i)}}{\sqrt{1+\rho_{\min}Hk_i/2}}
\end{align*}
with probability at least $1-2\delta_i$. Now, recall that at the end of the $i$-th epoch, \texttt{ARL-LIN(norm)} set the estimate of $\|\theta^*\|$ to
\begin{align*}
    b^{(i+1)} = \max_{\theta \in \mathcal{B}_{k_i}} \|\theta\|.
\end{align*}
From the definition of $\mathcal{B}_{k_i}$, we obtain
\begin{align*}
    b^{(i+1)} =\|\widehat{\theta}_{k_i}\| +  \frac{ \sqrt{\beta_{k_i}(\delta_i)}}{\sqrt{1+\rho_{\min}Hk_i/2}} 
     \leq \|\theta^*\| + 2 \frac{ \sqrt{\beta_{k_i}(\delta_i)}}{\sqrt{1+\rho_{\min}Hk_i/2}}
\end{align*}
with probability exceeding $1-2\delta_i$. Let us now look at the confidence radius
\begin{align*}
    \beta_{k_i}(\delta_i) = O\left((b^{(i)})^2 +  H^2d\log(k_iH)\log^2(k_i^2H/\delta_i)\right).
\end{align*}
We now substitute $k_i = 2^{i-1}k_1$ and $\delta_i = \frac{\delta}{2^{i-1}}$ to obtain
\begin{align*}
    \sqrt{1+\rho_{\min}Hk_i/2} \geq 2^{\frac{i-2}{2}} \sqrt{\rho_{\min}Hk_1}, \;\; \text{and}
\end{align*}
\begin{align*}
    \frac{\sqrt{\beta_{k_i}(\delta_i)}}{ \sqrt{1+\rho_{\min}Hk_i/2}} \leq \frac{C_1}{2}\frac{ b^{(i)}}{2^{\frac{i-2}{2}}\sqrt{\rho_{\min}Hk_1}} + \frac{C_2}{2} \frac{i^{3/2}}{2^{\frac{i-2}{2}}\sqrt{\rho_{\min}Hk_1}} \left(H\sqrt{d\log(k_1H)\log^2(k_1^2 H/\delta_1)}  \right)
\end{align*}
for some universal constants $C_1,C_2$.
Using this, we obtain
\begin{align*}
    b^{(i+1)} &\leq \|\theta^*\| + C_1\frac{ b^{(i)}}{2^{\frac{i-2}{2}}\sqrt{\rho_{\min}Hk_1}} + C_2 \frac{i^{3/2}}{2^{\frac{i-2}{2}}\sqrt{\rho_{\min}Hk_1}} \left(H\sqrt{d\log(k_1H)\log^2(k_1^2 H/\delta_1)}  \right) \\
    & = \|\theta^*\| + b^{(i)}\frac{p}{2^{\frac{i-2}{2}}}\sqrt{\frac{1}{k_1}} + \frac{i^{3/2}q}{2^{\frac{i-2}{2}}} \sqrt{\frac{d}{k_1}},
\end{align*}
with probability at least $1-2\delta_i$, where
\begin{align*}
    p = \frac{C_1}{\sqrt{\rho_{\min}H}} \,\,\,\,\text{and} \,\,\, q = \frac{C_2}{\sqrt{\rho_{\min}}} \sqrt{H\log(k_1H)\log^2(k_1^2 H/\delta_1)}   .
\end{align*}
Hence, we obtain, with probability at least $1-2\delta_i$,
\begin{align*}
    b^{(i+1)} - b^{(i)} \leq \|\theta^*\| - \left(1-\frac{p}{2^{\frac{i-2}{2}}}\sqrt{\frac{1}{k_1}}\right)b^{(i)} + \frac{i^{3/2}q}{2^{\frac{i-2}{2}}} \sqrt{\frac{d}{k_1}}.
\end{align*}
By construction, $b^{(i)} \geq \|\theta^*\|$ (since $\theta^* \in \cB_{k_i}$). Hence, provided $k_1 > \frac{4p^2}{2^i}$, we have
\begin{align*}
    b^{(i+1)} - b^{(i)} \leq \frac{p}{2^{\frac{i-2}{2}}}\sqrt{\frac{1}{k_1}} \|\theta^*\|  + \frac{i^{3/2}q}{2^{\frac{i-2}{2}}} \sqrt{\frac{d}{k_1}}
\end{align*}
with probability at least $1-2\delta_i$.
From the above expression, we have
\begin{align*}
    \sup_i b^{(i)} < \infty
\end{align*}
with probability greater than or equal to 
\begin{align*}
    1-\sum_i 2\delta_i = 1-\sum_i 2\delta/2^{i-1} = 1-4\delta.
\end{align*}
From the expression of $b^{(i+1)}$ and using the above fact in conjunction yield
\begin{align*}
    \lim_{i \rightarrow \infty} b^{(i)} \leq \|\theta^*\|.
\end{align*}
However, by construction $b^{(i)} \geq \|\theta^*\|$. Using this, along with the above equation, we obtain
\begin{align*}
    \lim_{i \rightarrow \infty} b^{(i)} = \|\theta^*\|.
\end{align*}
with probability exceeding $1-4\delta$. So, the sequence $\{b^{(1)},b^{(2)},...\}$ converges to $\|\theta^*\|$ with probability at least $1-4\delta$, and hence our successive refinement algorithm is consistent.

\paragraph{Rate of Convergence:} Since
\begin{align*}
    b^{(i+1)} - b^{(i)} = \Tilde{\mathcal{O}}\left( \frac{i^{3/2}}{2^i} \right),
\end{align*}
with high probability, the rate of convergence of the sequence $\{b^{(i)}\}_{i=1}^\infty$ is exponential in the number of epochs.

\paragraph{Uniform upper bound on $b^{(i)}$:}

We now compute a uniform upper bound on $b^{(i)}$ for all $i$. Consider the sequences $ \Bigg \{\frac{i^{3/2}}{2^{\frac{i-2}{2}}} \Bigg \}_{i=1}^\infty$ and $ \Bigg \{\frac{1}{2^{\frac{i-2}{2}}} \Bigg \}_{i=1}^\infty$, and let $t_j$ and $u_j$ denote the $j$-th term of the sequences respectively. It is easy to check that $\sup_{i}t_i < \infty $ and $\sup_{i}u_i < \infty$, and that the sequences $\{t_i\}_{i=1}^\infty$ and $\{u_i\}_{i=1}^\infty$ are convergent. With this new notation, we have
\begin{align*}
    b^{(2)} \leq \|\theta^*\| + u_1 \frac{p b^{(1)}}{\sqrt{k_1}} + t_1 \frac{q \sqrt{d}}{\sqrt{k_1}}
\end{align*}
with probability at least $1-2\delta$. Similarly, for $b^{(3)}$, we have
\begin{align*}
    b^{(3)} & \leq \|\theta^*\| + u_2 \frac{p b^{(2)}}{\sqrt{k_1}} + t_2 \frac{q \sqrt{d}}{\sqrt{k_1}} \\
    & \leq  \left (1+u_2 \frac{p}{\sqrt{k_1}} \right) \|\theta^*\| + \left( u_1 u_2 \frac{p}{\sqrt{k_1}} \frac{p}{\sqrt{k_1}} b^{(1)} \right) + \left( t_1 u_2 \frac{p}{\sqrt{k_1}} \frac{q \sqrt{d}}{\sqrt{k_1}} + t_2 \frac{q \sqrt{d}}{\sqrt{k_1}}\right)
\end{align*}
with probability at least $1 - 2\delta -\delta = 1-3\delta$. Similarly, we write expressions for $b^{(4)},b^{(5)},...$. Now, provided $k_1 \geq C \left ( \max\{p,q\}  \, b^{(1)} \right)^2 \, d$, where $C $ is a sufficiently large constant, the expression for $b^{(i)}$ can be upper-bounded as
\begin{align*}
    b^{(i)} \leq c_1 \| \theta^* \| + c_2
\end{align*}
for all $i$, where $c_1,c_2 > 0$ are some universal constants, which are obtained from summing an infinite geometric series with decaying step size. The above expression holds
with probability at least $1-\sum_{i} 2 \delta_i=1-4\delta$, which completes the proof.
\end{proof}

Armed with the above result, we finally focus on the regret bound for \texttt{ARL-LIN(norm)}.
\begin{theorem}[Cumulative regret of \texttt{ARL-LIN(norm)}]
\label{thm:norm}
Fix any $\delta \in (0,1]$, and suppose that the hypothesis of Lemma \ref{lem:sequence} holds. Then, with probability exceeding $1- 6\delta$, \texttt{ARL-LIN(Norm)} enjoys the regret bound
\begin{align*}
    R(T) \!=\!  \Tilde{\mathcal{O}}\!\left(\!\| \theta^* \|d\sqrt{H^3T} \log(k_1H)\log(k_1^2H/\delta) \!\right),
\end{align*}
where $T=KH$ denotes the total number of rounds, and $\Tilde{\mathcal{O}}$ hides a $\mathsf{polylog}(K/k_1)$ factor.
\end{theorem}

\begin{proof}
The cumulative regret is given by
\begin{align*}
    R(T) \leq \sum_{i=1}^N R^{\texttt{UCRL-VTR-LIN}}(k_i,\delta_i,b^{(i)}),
\end{align*}
where $N$ denotes the total number of epochs $R^{\texttt{UCRL-VTR-LIN}}(k_i,\delta_i,b^{(i)})$ denotes the cumulative regret of \texttt{UCRL-VTR-LIN}, when it is run with confidence level $\delta_i$ and norm upper bound $b^{(i)}$ for $k_i$ episodes. Using the result of
\citet{jia2020model}, we have
\begin{align*}
    R^{\texttt{UCRL-VTR-LIN}}(k_i,\delta_i,b^{(i)}) =\mathcal{O}\left(b_idH^2\sqrt{k_i}\log(k_iH)\log(k_i^2H/\delta_i) \right)
\end{align*}
with probability at least $1-\delta_i$. Now, using Lemma~\ref{lem:sequence}, we obtain
\begin{align*}
    R(T) \le (c_1 \| \theta^* \| + c_2) \sum_{i=1}^N \mathcal{O}\left(dH^2\sqrt{k_i}\log(k_iH)\log(k_i^2H/\delta_i) \right)
\end{align*}
with probability at least $1-4\delta-\sum_{i}\delta_i$. Substituting $k_i = 2^{i-1} k_1$ and $\delta_i = \frac{\delta}{2^{i-1}}$, we obtain
\begin{align*}
    R(T) \leq  (c_1 \| \theta^* \| + c_2) \sum_{i=1}^N \mathcal{O}\left(dH^2 \sqrt{k_i} \,\, \mathsf{poly}(i)\,\,  \log(k_1H)\log(k_1^2H/\delta) \right)
\end{align*}
with probability at least $1-4\delta -2\delta=1-6\delta$.
Using the above expression, we obtain
\begin{align*}
    R(T) & \leq  (c_1 \| \theta^* \| + c_2)  \mathcal{O}\left(dH^2  \log(k_1H)\log(k_1^2H/\delta) \right) \sum_{i=1}^N \mathsf{poly}(i)\sqrt{k_i}\\
    & \leq (c_1 \| \theta^* \| + c_2)  \mathcal{O}\left(dH^2  \log(k_1H)\log(k_1^2H/\delta) \right) \mathsf{poly}(N)\,\, \sum_{i=1}^N  \sqrt{k_i} \\
    & \leq   (c_1 \| \theta^* \| + c_2)  \mathcal{O}\left(dH^2  \log(k_1H)\log(k_1^2H/\delta) \right) \mathsf{polylog}(K/k_1)\,\, \sum_{i=1}^N  \sqrt{k_i} \\
    & \leq  (c_1 \| \theta^* \| + c_2)  \mathcal{O}\left(dH^2  \log(k_1H)\log(k_1^2H/\delta) \right) \mathsf{polylog}(K/k_1) \sqrt{K}\\
    & = \mathcal{O}\left(\| \theta^* \| d\sqrt{H^3 T} \log(k_1H)\log(k_1^2H/\delta)  \mathsf{polylog}(K/k_1) \right),
\end{align*}
where we have used that $N=\cO\left(\log (K/k_1)\right)$, $\sum_{i=1}^N \sqrt{k_i} = \cO(\sqrt{K})$, and $T=KH$.
The above regret bound holds with probability greater than or equal to $1- 6\delta$, which completes the proof.
\end{proof}

\end{document}